\documentclass{article}

\usepackage{microtype}
\usepackage{graphicx}
\usepackage{subfigure}
\usepackage{booktabs} 
\usepackage{multirow}

\usepackage{styles/icml2023/algorithm}
\usepackage[noend]{styles/icml2023/algorithmic}

\newcommand{\alglinelabel}{%
  \addtocounter{ALC@line}{-1}
  \refstepcounter{ALC@line}
  \label
}

\usepackage[dvipsnames]{xcolor}
\definecolor{sizzurp}{rgb}{0.8, 0.0, 0.8}

\usepackage{styles/preprint/arxiv}

\usepackage{natbib}
\bibpunct{(}{)}{;}{a}{,}{,}

\usepackage[
    colorlinks,
    citecolor=blue!70!black,
    linkcolor=blue!70!black,
    urlcolor=blue!70!black,
    breaklinks=true
]{hyperref}

\usepackage[utf8]{inputenc}
\usepackage[T1]{fontenc}
\usepackage{amsmath, amsthm, amssymb}

\usepackage{url}
\usepackage{packages/breakcites}

\usepackage{booktabs}
\usepackage{wrapfig}
\usepackage{graphics}
\graphicspath{{./figs/}}
\usepackage{listings}
\usepackage{bm}
\usepackage{bbm}
\usepackage{enumitem}
\usepackage{nicefrac}
\usepackage{mathrsfs}
\usepackage{inconsolata}
\usepackage{comment}
\usepackage{xspace}
\usepackage{transparent}
\usepackage[nameinlink, capitalize]{cleveref}
\usepackage{thmtools}
\usepackage{thm-restate}
\usepackage{nicematrix}
\usepackage{mathtools}
\usepackage{lineno}

\usepackage{packages/color-edits}
\addauthor{yz}{purple}
\addauthor{paul}{blue}
\addauthor{mr}{ForestGreen}

\def\E{{\mathbb E}}

\def\P{{\mathbb P}}
\def\R{{\mathbb R}}

\def\S{{\mathbb S}}

\DeclareMathOperator*{\argmin}{arg\,min}

\let\oldparagraph\paragraph
\renewcommand{\paragraph}[1]{\oldparagraph{#1.}}

\renewcommand{\epsilon}{\varepsilon}

\newcommand{\ind}{\mathbbm{1}}

\newcommand{\ldef}{\vcentcolon=}

\newcommand{\bigoh}{O}
\newcommand{\bigoht}{\wt{O}}
\newcommand{\bigom}{\Omega}
\newcommand{\bigomt}{\wt{\Omega}}

\newcommand{\polylog}{\mathrm{polylog}}
\renewcommand{\epsilon}{\varepsilon}
\newcommand{\indic}{\mathbb{I}}

\newtheorem{lemma}{Lemma}

\newtheorem{assumption}{Assumption}

\newcommand{\algcommentlight}[1]{\textcolor{blue!70!black}{\transparent{0.5}\small{\texttt{\textbf{//\hspace{2pt}#1}}}}}

\DeclarePairedDelimiter{\abs}{\lvert}{\rvert} %
\DeclarePairedDelimiter{\brk}{[}{]}
\DeclarePairedDelimiter{\crl}{\{}{\}}
\DeclarePairedDelimiter{\prn}{(}{)}

\DeclarePairedDelimiterX{\infdiv}[2]{(}{)}{%
  #1\;\delimsize\|\;#2%
}

\newcommand{\wt}[1]{\widetilde{#1}}
\newcommand{\wh}[1]{\widehat{#1}}

\def\ddefloop#1{\ifx\ddefloop#1\else\ddef{#1}\expandafter\ddefloop\fi}
\def\ddef#1{\expandafter\def\csname bb#1\endcsname{\ensuremath{\mathbb{#1}}}}
\ddefloop ABCDEFGHIJKLMNOPQRSTUVWXYZ\ddefloop
\def\ddefloop#1{\ifx\ddefloop#1\else\ddef{#1}\expandafter\ddefloop\fi}
\def\ddef#1{\expandafter\def\csname b#1\endcsname{\ensuremath{\mathbf{#1}}}}
\ddefloop ABCDEFGHIJKLMNOPQRSTUVWXYZ\ddefloop
\def\ddef#1{\expandafter\def\csname sf#1\endcsname{\ensuremath{\mathsf{#1}}}}
\ddefloop ABCDEFGHIJKLMNOPQRSTUVWXYZ\ddefloop
\def\ddef#1{\expandafter\def\csname c#1\endcsname{\ensuremath{\mathcal{#1}}}}
\ddefloop ABCDEFGHIJKLMNOPQRSTUVWXYZ\ddefloop
\def\ddef#1{\expandafter\def\csname h#1\endcsname{\ensuremath{\widehat{#1}}}}
\ddefloop ABCDEFGHIJKLMNOPQRSTUVWXYZ\ddefloop
\def\ddef#1{\expandafter\def\csname hc#1\endcsname{\ensuremath{\widehat{\mathcal{#1}}}}}
\ddefloop ABCDEFGHIJKLMNOPQRSTUVWXYZ\ddefloop
\def\ddef#1{\expandafter\def\csname t#1\endcsname{\ensuremath{\widetilde{#1}}}}
\ddefloop ABCDEFGHIJKLMNOPQRSTUVWXYZ\ddefloop
\def\ddef#1{\expandafter\def\csname tc#1\endcsname{\ensuremath{\widetilde{\mathcal{#1}}}}}
\ddefloop ABCDEFGHIJKLMNOPQRSTUVWXYZ\ddefloop
\def\ddefloop#1{\ifx\ddefloop#1\else\ddef{#1}\expandafter\ddefloop\fi}
\def\ddef#1{\expandafter\def\csname scr#1\endcsname{\ensuremath{\mathscr{#1}}}}
\ddefloop ABCDEFGHIJKLMNOPQRSTUVWXYZ\ddefloop

\renewcommand{\bigm}[1]{%
  \ifcsname fenced@\string#1\endcsname
    \expandafter\@firstoftwo
  \else
    \expandafter\@secondoftwo
  \fi
  {\expandafter\amsmath@bigm\csname fenced@\string#1\endcsname}%
  {\amsmath@bigm#1}%
}

\newcommand{\DeclareFence}[2]{\@namedef{fenced@\string#1}{#2}}
\makeatother

\makeatletter
\let\save@mathaccent\mathaccent
\newcommand*\if@single[3]{%
  \setbox0\hbox{${\mathaccent"0362{#1}}^H$}%
  \setbox2\hbox{${\mathaccent"0362{\kern0pt#1}}^H$}%
  \ifdim\ht0=\ht2 #3\else #2\fi
  }
\newcommand*\rel@kern[1]{\kern#1\dimexpr\macc@kerna}
\newcommand*\widebar[1]{\@ifnextchar^{{\wide@bar{#1}{0}}}{\wide@bar{#1}{1}}}
\newcommand*\wide@bar[2]{\if@single{#1}{\wide@bar@{#1}{#2}{1}}{\wide@bar@{#1}{#2}{2}}}
\newcommand*\wide@bar@[3]{%
  \begingroup
  \def\mathaccent##1##2{%
    \let\mathaccent\save@mathaccent
    \if#32 \let\macc@nucleus\first@char \fi
    \setbox\z@\hbox{$\macc@style{\macc@nucleus}_{}$}%
    \setbox\tw@\hbox{$\macc@style{\macc@nucleus}{}_{}$}%
    \dimen@\wd\tw@
    \advance\dimen@-\wd\z@
    \divide\dimen@ 3
    \@tempdima\wd\tw@
    \advance\@tempdima-\scriptspace
    \divide\@tempdima 10
    \advance\dimen@-\@tempdima
    \ifdim\dimen@>\z@ \dimen@0pt\fi
    \rel@kern{0.6}\kern-\dimen@
    \if#31
      \overline{\rel@kern{-0.6}\kern\dimen@\macc@nucleus\rel@kern{0.4}\kern\dimen@}%
      \advance\dimen@0.4\dimexpr\macc@kerna
      \let\final@kern#2%
      \ifdim\dimen@<\z@ \let\final@kern1\fi
      \if\final@kern1 \kern-\dimen@\fi
    \else
      \overline{\rel@kern{-0.6}\kern\dimen@#1}%
    \fi
  }%
  \macc@depth\@ne
  \let\math@bgroup\@empty \let\math@egroup\macc@set@skewchar
  \mathsurround\z@ \frozen@everymath{\mathgroup\macc@group\relax}%
  \macc@set@skewchar\relax
  \let\mathaccentV\macc@nested@a
  \if#31
    \macc@nested@a\relax111{#1}%
  \else
    \def\gobble@till@marker##1\endmarker{}%
    \futurelet\first@char\gobble@till@marker#1\endmarker
    \ifcat\noexpand\first@char A\else
      \def\first@char{}%
    \fi
    \macc@nested@a\relax111{\first@char}%
  \fi
  \endgroup
}
\makeatother

\newcommand{\RegSq}{\mathrm{\mathbf{Reg}}_{\mathsf{Sq}}}

\newcommand{\AlgSq}{\mathrm{\mathbf{Alg}}_{\mathsf{Sq}}}

\newcommand{\AlgSample}{\mathrm{\mathbf{Alg}}_{\mathsf{Sample}}}
\newcommand{\dec}{\mathsf{dec}}

\newcommand{\regsq}{\RegSq}

\newcommand{\sqalgtext}{$\AlgSq$\xspace}

\newcommand{\regcbt}{\mathrm{\mathbf{Reg}}_{\mathsf{CB},\tau}}

\newcommand{\samplealg}{\AlgSample}
\newcommand{\samplealgtext}{$\samplealg$\xspace}

\newcommand{\smthigw}{\textsf{SmoothIGW}\xspace}

\newcommand{\sgm}{\Omega}
\newcommand{\smtht}{\mathsf{Smooth}_\tau\xspace}

\newcommand{\smth}{\text{smooth}\xspace}

\newcommand{\dectext}{Decision-Estimation Coefficient\xspace}

\newcommand{\cappedIGW}{\textsf{CappedIGW}\xspace}

\newcommand{\gfuncdenomtwoarg}[2]{1 + \gamma \left(\widehat f_t(x_t, #2) - #1\right)_+}
\newcommand{\gfuncdenom}[1]{\gfuncdenomtwoarg{#1}{a_t}}
\newcommand{\gfunc}[1]{\frac{\tau}{\gfuncdenom{#1}}}
\newcommand{\gfunctwoarg}[2]{\frac{\tau}{\gfuncdenomtwoarg{#1}{#2}}}

\title{Infinite Action Contextual Bandits with Reusable Data Exhaust}

\author{Mark Rucker \\
	University of Virginia\\
	\And
	Yinglun Zhu \\
	University of Wisconsin--Madison\\
        \And
	Paul Mineiro \\
	Microsoft Research NYC\\
}

\renewcommand{\shorttitle}{Infinite Action Contextual Bandits with Reusable Data Exhaust}

\hypersetup{
pdftitle={Infinite Action Contextual Bandits with Reusable Data Exhaust},
pdfauthor={Mark Rucker, Yinglun Zhu, Paul Mineiro}
}

\begin{document}

\maketitle

\begin{abstract}
For infinite action contextual bandits,
smoothed regret and reduction to regression
results in state-of-the-art online
performance with computational cost
independent of the action set:
unfortunately, the resulting data exhaust
does not have well-defined importance-weights.
This frustrates the execution of downstream data
science processes such as offline model selection. In this paper we describe an
online algorithm with an equivalent
smoothed regret guarantee, but which generates
well-defined importance weights: in exchange,
the online computational cost increases, but
only to order smoothness (i.e., still independent
of the action set). This removes a key
obstacle to adoption of smoothed regret
in production scenarios.
\end{abstract}

\section{Introduction}

Those who ignore history are doomed to
repeat it. A modern variant of this truth
arises in controlled experimentation platforms,
where offline procedures are a critical
complement to online tests, e.g.,
supporting counterfactual evaluation strategies~\citep{agarwal2016making},
offline model selection~\citep{li2015counterfactual},
and prioritization of scarce online experimental
resources~\citep{gomez2015netflix}.
Consequently, the utility of a learning
algorithm is not solely determined by online
performance, but also by the post-hoc utility
of the data exhaust.

The recent contribution of \citet{zhu2022contextual} exemplifies this:
an online contextual bandit algorithm
for infinite action spaces with $O(1)$
space and time complexity with respect to the
action set.  Unfortunately, this performance is achieved
by sampling from a distribution which is not
absolutely continuous with the reference
measure. Therefore, a variety of post-hoc
evaluation procedures that rely on importance-weighting
cannot be applied, limiting adoption.

In this paper, we describe an alternative approach
to infinite action spaces which not only enjoys
similar smooth regret guarantee (and empirical performance), but also utilizes sampling
distributions with well defined
importance-weights.  In exchange, we pay
an increased computational cost. However, the computational cost only scales
with the smoothness of the regret guarantee,
rather than the cardinality or dimensionality
of the action space per se.  Furthermore
the new approach does not require an
$\argmin$ oracle, which plays a critical role in the work of  \citet{zhu2022contextual}.

\paragraph
{Contributions}
We highlight our main contributions:
\begin{enumerate}
\item In \cref{subsec:cappedigw}, we present \cappedIGW, an algorithm that achieves near-optimal smooth regret guarantees with (i) a sampling distribution that generates reusable data exhaust, and (ii) no dependence on an expansive $\argmin$ oracle (which is used by previous algorithms).

\item In \cref{subsec:normcs}, we develop algorithms to
efficiently implement the algorithm \cappedIGW.
Our computational complexity only scales with the smoothness
parameter, but otherwise has no explicit dependence
on the cardinality or dimensionality of the action space.
Our implementation leverages
techniques from betting
martingales~\citep{waudby2020estimating}
and is of independent interest for Monte-Carlo
integration.
\end{enumerate}

In \cref{sec:experiment}, we provide
experimental demonstrations exhibiting
a combination of equivalent online performance
to \citet{zhu2022contextual} and superior
offline utility.

\section{Problem Setting}
\label{sec:setting}

Unfortunately several unusual aspects of our
approach demand a tedious exposition:
we operate via reduction to regression;
we use a nonstandard (smoothed) regret criterion;
and our computational complexity claims
require careful specification of oracles
in the infinite action setting.  The impatient
reader can skip directly to \cref{sec:statistical} and
use this section as reference.

\paragraph{Notation} For functions
$f,g:\cZ\to \R_{+}$, we write $f=\bigoh(g)$ (resp. $f=\bigom(g)$) if there exists a constant
$C>0$ such that $f(z)\leq{}Cg(z)$ (resp. $f(z)\geq{}Cg(z)$)
for all $z\in\cZ$. We write $f=\bigoht(g)$ if
$f=\bigoh(g\cdot\mathrm{polylog}(T))$, $f=\bigomt(g)$ if $f=\bigom(g/\polylog(T))$.
        For a set $\cZ$, we let
        $\Delta(\cZ)$ denote the set of all Radon probability measures
        over $\cZ$.
	We let $\indic_{z}\in\Delta(\cZ)$ denote the delta distribution on $z$.
    For $x \in \mathbb{R}$ we define $(x)_+ \ldef \max\left(x, 0\right)$.

\subsection{Contextual Bandits: Reduction to regression}

We consider the following standard contextual bandit problems. At any time step $t \in [T]$, nature selects a context $x_t \in \cX$ and a distribution over loss functions $\ell_t: \cA \rightarrow [0,1]$ mapping from the (compact) action set $\cA$ to a loss value in $[0, 1]$.  Conditioned on the context $x_t$, the loss function is stochastically generated, i.e., $\ell_t \sim \P_{\ell_t}(\cdot \mid x_t)$.
The learner selects an action $a_t \in \cA$ based on the revealed context $x_t$, and obtains (only) the loss $\ell_t(a_t)$ of the selected action.
The learner has access to a set of measurable regression functions $\cF \subseteq (\cX \times \cA \rightarrow [0,1])$ to predict the loss of any context-action pair.
We make the following standard realizability assumption studied in the contextual bandit literature \citep{agarwal2012contextual, foster2018practical, foster2020beyond, simchi2021bypassing}.
\begin{assumption}[Realizability]
\label{asm:realizability}
There exists a regression function $f^\star \in \cF$ such that $ \E \brk{\ell_t(a) \mid x_t} = f^\star(x_t, a)$ for any $a \in \cA$ and across all $t \in [T]$.
\end{assumption}

\subsection{Smoothed Regret}
\label{sec:smooth_regret}
Let $(\cA, \sgm)$ be a measurable space of the action set and $\mu$ be a base probability measure over the actions.
Let $\cQ_\tau$ denote the set of probability measures such that, for any measure $Q \in \cQ_\tau$, the following holds true: (i) $Q$ is absolutely continuous with respect to the base measure  $\mu$, i.e., $Q \ll \mu$; and (ii) The Radon-Nikodym derivative of $Q$ with respect to $\mu$ is no larger than $\tau$, i.e., $\frac{dQ}{d\mu} \leq \tau$.
We call $\cQ_\tau$ the set of {smoothing kernels at smoothness level $\tau$, or simply put the set of $\tau$-smoothed kernels.}
For any context $x \in \cX$, we denote by $\smtht(x)$ the smallest loss incurred by any  $\tau$-smoothed kernel, i.e.,
\begin{align*}
    \smtht(x) \ldef \inf_{Q \in \cQ_\tau}\E_{a \sim Q}\brk{f^\star(x,a)}.
\end{align*}
Rather than competing with $\argmin_{a\in\cA}f^\star(x,a)$---which is minimax prohibitive in infinite action spaces---
we take $\smtht(x)$ as the benchmark and define the \emph{smooth regret} as follows:
\begin{align}
    \regcbt(T) & \coloneqq
    \E \brk*{ \sum_{t=1}^T  f^\star(x_t, a_t) - \smtht(x_t) } \label{eqn:smooth_regret}.
\end{align}
One important feature about the above definition is that the benchmark, i.e., $\smtht(x_t)$, automatically adapts to the context $x_t$:
this gives the benchmark more power and makes it harder to compete against, compared to previously studied baselines \citep{chaudhuri2018quantile, krishnamurthy2020contextual}.

\subsection{Computational Oracles}
The first step towards designing computationally efficient algorithms is to identify reasonable oracle models to access the sets of regression functions or actions. Otherwise, enumeration over regression functions or actions (both can be exponentially large) immediately invalidate the computational efficiency.
We consider two common oracle models: a regression oracle and a sampling oracle.

\paragraph{The regression oracles}
A fruitful approach to designing efficient contextual bandit algorithms is through reduction to supervised regression with the class $\cF$ \citep{foster2020beyond, simchi2021bypassing, foster2020adapting, foster2021instance}.
We provide a brief introduction to the reduction technique employed in this paper in \cref{sec:background}.
Following \citet{foster2020beyond}, we assume that we have access to an \emph{online} regression oracle \sqalgtext, which is an algorithm for sequential prediction under square loss.
More specifically, the oracle operates in the following protocol: At each round $t \in [T]$, the oracle makes a prediction $\wh f_t$, then receives context-action-loss tuple $(x_t, a_t, \ell_t(a_t))$.
The goal of the oracle is to accurately predict the loss as a function of the context and action, and we evaluate its performance via the square loss $\prn{\wh f_t(x_t,a_t) - \ell_t(a_t)}^2$.
We measure the oracle's cumulative performance through the square-loss regret to $\cF$, which is formalized below.

\begin{assumption}
\label{assumption:regression_oracle}
The regression oracle \sqalgtext guarantees that, with probability at least $1-\delta$, for any (potentially adaptively chosen) sequence $\crl*{(x_t, a_t, \ell_t(a_t))}_{t=1}^T$,
\begin{equation*}
	\E \Bigg[\sum_{t=1}^T  \prn*{\wh f_t(x_t, a_t) - \ell_t(a_t)}^2 \Bigg.
     \Bigg. - \inf_{f \in \cF} \sum_{t=1}^T \prn*{f(x_t, a_t) - \ell_t(a_t)}^2 \Bigg]
    \leq \regsq(T, \delta),
\end{equation*}
for some (non-data-dependent) function $\regsq(T, \delta)$.
\end{assumption}

We will consider the following operations $O(1)$ cost: (i) query the oracle's estimator $\wh f_t$ with context-action pair $(x_t,a)$ and receive its predicted value $\wh f_t(x_t,a) \in [0,1]$;  and (ii) update the oracle with example $(x_t, a_t, \ell_t(a_t))$.

Online regression is a well-studied problem, with known algorithms for many model classes \citep{foster2020beyond, foster2020adapting}: including linear models \citep{hazan2007logarithmic}, generalized linear models \citep{kakade2011efficient}, non-parametric models \citep{gaillard2015chaining}, and beyond.
Using Vovk's aggregation algorithm \citep{vovk1998game}, one can show that $\regsq(T, \delta) = O(\log \prn{\abs{\cF} /\delta})$ for any finite set of regression functions $\cF$, which is the canonical setting studied in contextual bandits \citep{langford2007epoch, agarwal2012contextual}.
In the following of this paper, we use abbreviation $\regsq(T) \ldef \regsq(T, T^{-1})$, and will keep the $\regsq(T)$ term in our regret bounds to accommodate for general set of regression functions.

\paragraph{The sampling oracle}
In order to design algorithms that work with
large/continuous action spaces, we assume access
to a sampling oracle \samplealgtext to get access
to the action space.  In particular, the oracle
\samplealgtext returns an action $a \sim \mu$
randomly drawn according to the base probability
measure $\mu$ over the action space $\cA$.  We
consider this operation $O(1)$ cost.

\paragraph{Representing the actions} In practice
the number of bits required to represent any action
$a\in \cA$ scales with $O(\log\abs{\cA})$ with a
finite set of actions and $\wt O(d)$ for actions represented as vectors in $\R^d$.  Nonetheless
we consider this $O(1)$, i.e., we elide
the representational overhead in big-$\bigoh$
notation for our computational analysis.

\section{Algorithms}
\label{sec:statistical}

\subsection{Background: \smthigw}
\label{sec:statistical_background}
\citet{zhu2022contextual} designed an oracle-efficient \smthigw that achieves a $\sqrt{T}$-type regret under the \smth regret defined in \cref{eqn:smooth_regret}.
\cref{alg:smooth} contains the pseudo code of
the \smthigw algorithm. At each round
$t \in [T]$, the learner observes the
context $x_t$ from the environment, obtains
the estimator $\widehat f_t$ from the
regression oracle \sqalgtext, and computes
the greedy action $\wh a_t$. It then constructs
a sampling distribution $P_t$ by mixing a
smoothed inverse gap weighted (IGW)
distribution~\citep{abe1999associative, foster2020beyond}
and a delta mass at the greedy action. The
algorithm samples an action $a_t \sim P_t$ and
updates the regression oracle.
\begin{algorithm}[h]
	\caption{\smthigw \citep{zhu2022contextual}}
	\label{alg:smooth}
	\renewcommand{\algorithmicrequire}{\textbf{Input:}}
	\renewcommand{\algorithmicensure}{\textbf{Output:}}
	\newcommand{\algorithmicbreak}{\textbf{break}}
    \newcommand{\BREAK}{\STATE \algorithmicbreak}
	\begin{algorithmic}[1]
		\REQUIRE Exploration parameter $\gamma > 0$; online regression oracle \sqalgtext.
		\FOR{$t = 1, 2, \dots, T$}
		\STATE Observe context $x_t$.
		\STATE Receive $\widehat f_t$ from regression oracle \sqalgtext.
		\STATE Get $\widehat a_t \ldef \argmin_{a \in \cA} \widehat f_t(x_t, a)$. \alglinelabel{alg:smooth:argmin}
		\STATE Set
		\begin{align*}
		P_t \ldef M_t +  (1-M_t(\cA)) \cdot \indic_{\widehat a_t}
		\end{align*}
		where {$M_t$ is the measure defined in \cref{eqn:smoothedabelong}} \alglinelabel{alg:smooth:pt}
		\STATE Sample $a_t \sim P_t$ and observe loss $\ell_t(a_t)$.
        \STATE  Update \sqalgtext with $(x_t, a_t, \ell_t(a_t))$
		\ENDFOR
	\end{algorithmic}
\end{algorithm}

The measure $M_t$ on line \ref{alg:smooth:pt} of \cref{alg:smooth} is defined by the following density with respect to the reference measure,
\begin{align}
    \frac{dM_t}{d\mu}(a) \ldef \frac{\tau}{\tau+  \gamma \cdot \left(\widehat f_t(x_t, a) - \wh f_t(x_t, \wh a_t)\right) }.
    \label{eqn:smoothedabelong}
\end{align}
Note that $M_t$ is only a sub-probability measure since  $d M_t / d\mu (a) \leq 1$, hence an additional $(1 - M_t(\cA) ) \cdot \ind_{\wh a_t}$ term is needed (to make sure that $P_t$ is a probability measure).

\paragraph{The Problems}
While \smthigw is the first oracle-efficient contextual bandit algorithm that works with smooth regret, it is not without problems.
We highlight two problems associated with \smthigw below.
\begin{itemize}
	\item \textbf{The $\argmin$ oracle.}
	Note that \cref{alg:smooth} requires an exact $\argmin$ oracle to compute the greedy action $\wh a_t$ (on line \ref{alg:smooth:argmin}), which is later on used to construct the sampling distribution $P_t$ (on line \ref{alg:smooth:pt}). However, when working with large, and potentially continuous, action spaces, it can be computationally expensive to obtain such an exact $\argmin$ oracle.  For their experiments, \citet{zhu2022contextual} construct a regressor class with an $O(1)$ $\argmin$ oracle, but their construction induces a unimodal $\widehat f_t$, which may not always be appropriate.

	 \item \textbf{Insufficient data reuse.}
	While the $M_t$ term is always absolutely continuous with respect to the base measure  $\mu$, the delta distribution $\ind_{\wh a_t}$ is not absolutely continuous with respect to $\mu$ in many common cases, e.g., when $\mu$ is the Lebesgue measure in $\R_d$. As a result, a variety of post-hoc procedures that rely on importance-weighting cannot be applied.  Unfortunately, to achieve $O(\sqrt{T})$ regret, \smthigw uses $\gamma \propto \sqrt{T}$, which implies that the fraction of actions sampled from the $\ind_{\wh a_t}$ component increases with horizon length, e.g., \cref{fig:fractionofindic}.
\end{itemize}

These two drawbacks frustrate the deployment of \smthigw in real-world applications.

\subsection{New Approach: \cappedIGW}
\label{subsec:cappedigw}

Resolution of the above issues requires
eliminating the use of the greedy action
$\wh a_t$, which occurs in two places:
\begin{itemize}
	\item \textbf{Inverse-gap weighting.}
	In sub-probability measure $M_t$, its density (with respect to $\mu$) on any action  $a$ is defined to be inversely proportional to the empirical loss gap  $\prn*{\wh f_t(x_t, a) - \wh f_t(x_t, \wh a_t)}$: here, we use  $\wh f_t(x_t, \wh a_t)$ as a \emph{benchmark} to compute the loss gap.

	\item \textbf{Pseudo normalization.}
	Since $M_t$ is only a sub-probability measure, to actually sample from a probability measure, \smthigw shifts the remaining probability mass to the delta distribution at the greedy action, i.e., $\ind_{\wh a_t}$: here, we use  $\ind_{\wh a_t}$ to pseudo normalize the sub-probability measure  $M_t$.

\end{itemize}

In the sequel we eliminate use of the greedy action.

\paragraph{Sampling Density}
Let $\beta_t, \kappa_t \in \R$ be two parameters (whose values will be computed later).
We consider a probability measure $P_t$ whose density with respect to the base measure  $\mu$ is defined as follows:
\begin{equation}
\frac{dP_t}{d\mu}(a) = \kappa_t \gfunc{\beta_t},
\label{eqn:cappeddensity}
\end{equation}
where $(x)_+ \ldef \max\left(x, 0\right)$.
Relative to \smthigw:
\begin{itemize}
	\item
We replace the old loss gap benchmark $\wh f_t(x_t, \wh a_t)$ by the new parameter $\beta_t$; we also take another $\max$ operation over  $\wh f_t(x_t,a) - \beta_t$ and $0$ to ensure the positivity of the loss gap.
This was inspired by observing the optimal
$\tau$-smooth policy plays uniformly over
the $\tau^{-1}$-th quantile of the true $f^*$,
but is ultimately justified by the regret
decomposition in the proof of \cref{thm:regret}.
\item
We use $\kappa_t$ as a normalization factor instead of shifting mass to $\ind_{\wh a_t}$; the normalization factor $\kappa_t$ is
determined by the choice of $\beta_t$ via
\begin{align*}
	\kappa_t = 1 \Bigg/
\mathbb{E}_{a_t \sim \mu}\left[\gfunc{\beta_t}\right].
\end{align*}
\end{itemize}
With this new sampling distribution in \cref{eqn:cappeddensity} at hand, we develop a new algorithm for smooth regret, shown next in \cref{alg:capped}.

\begin{algorithm}[H]
	\caption{\cappedIGW}
	\label{alg:capped}
	\renewcommand{\algorithmicrequire}{\textbf{Input:}}
	\renewcommand{\algorithmicensure}{\textbf{Output:}}
	\newcommand{\algorithmicbreak}{\textbf{break}}
    \newcommand{\BREAK}{\STATE \algorithmicbreak}
	\begin{algorithmic}[1]
		\REQUIRE Exploration parameter $\gamma > 0$; online regression oracle \sqalgtext.
		\FOR{$t = 1, 2, \dots, T$}
		\STATE Observe context $x_t$.
		\STATE Receive $\widehat f_t$ from regression oracle \sqalgtext.
		\STATE Compute $\beta_t$. \hfill \algcommentlight{\cref{alg:normcs}}
        \STATE Sample $a_t \sim P_t$ \hfill \algcommentlight{\cref{eqn:cappeddensity}, \cref{alg:highdsampler}}
		\STATE Observe loss $\ell_t(a_t)$.
        \STATE Update \sqalgtext with $(x_t, a_t, \ell_t(a_t))$
		\ENDFOR
	\end{algorithmic}
\end{algorithm}

We will show in next section that $\beta_t$ can be computed efficiently in $\bigoht(\tau \log \tau)$ calls to the sampling oracle.
First we state a regret guarantee.
\begin{restatable}{theorem}{thmRegret}
\label{thm:regret}
Fix any smoothness level $\tau \geq 1$.
Suppose $\forall t: \kappa_t \geq 1$ and let $\kappa_\infty$ be an upper bound on  $\kappa_t$ for $\forall t$.
By setting the exploration parameter $\gamma = \sqrt{8 T \kappa_\infty \tau / \regsq(T)}$, \cref{alg:capped} ensures that
\begin{align*}
    \regcbt(T) \leq {\sqrt{4 T \, \tau \kappa_{\infty} \regsq(T)} }.
\end{align*}
\end{restatable}
\begin{proof} See \cref{app:proofthmregret}.
\end{proof}
The guarantee in \cref{thm:regret} is the same as the guarantee for \smthigw (which is near-optimal) up to a $\sqrt{\kappa_\infty}$ factor.
Since we can always find appropriate $\beta_t, \kappa_t$ to ensure $\kappa_\infty = O(1)$, we can efficiently achieve the near-optimal smooth regret guarantees without (i) an $\argmin$ oracle, and (ii) with full data exhaust reuse.

\paragraph{Adapting to an unknown smoothness level $\tau$} We can simply replace \smthigw with \cappedIGW in \citet[Thm. 2]{zhu2022contextual} to build (i) Pareto optimal algorithms with unknown smoothness level $\tau$, and (ii) develop nearly minimax optimal algorithms under the standard regret for bandits with multiple best arms \citet{zhu2020regret} and Lipschitz/H\"older bandits \citet{kleinberg2004nearly, hadiji2019polynomial}: see Section 4 and Section 5 in \citet{zhu2022contextual} for details.

\subsection{Efficient Implementation}
\label{subsec:normcs}
In this section, we discuss how to efficiently (i) compute parameter $\beta_t$ and (ii) sample actions from the distribution $P_t$.
We first notice that the condition $\kappa_t \geq 1$ is critical to \cref{thm:regret}.
Intuitively, $\beta_t$ must be chosen so that
\cref{alg:capped} plays a policy which is at
most $\tau$-smooth. Because we are competing
with $\tau$-smooth policies, it makes sense to be
less smooth than the competitor but not to
be more smooth than the competitor (further, as described at the end of Section~\ref{subsec:cappedigw}, the appropriate level for $\tau$ can be adaptively chosen).

Consistent with
\cref{thm:regret}, our task is to find a $\beta_t$
such that
\begin{equation}
\mathbb{E}_{a_t \sim \mu}\left[\gfunc{\beta_t}\right] \in \left[ \frac{1}{\kappa_{\infty}}, 1 \right].
\label{eqn:findbeta}
\end{equation}
First, we establish that it is provably possible
to satisfy \cref{eqn:findbeta} with high
probability using $O\left(\tau \log(\nicefrac{(\tau + \gamma)}{\delta})\right)$
samples from the reference measure.

\begin{restatable}{theorem}{thmSample}
\label{thm:sample}
With the choice $\kappa_{\infty} = 24$, with
probability at least $(1 - \delta)$,
it is possible to estimate $\beta$ satisfying
\cref{eqn:findbeta} using
$O\left(\tau \log(\nicefrac{(\tau + \gamma)}{\delta})\right)$
samples from $\mu$.
\end{restatable}
\begin{proof} See \cref{app:proofthmsample}
\end{proof}

\cref{thm:sample} uses a fixed sampling strategy
which  is amenable to analysis and provably
terminates after $\wt O(\tau \log \tau)$ samples.
However, this fixed sampling strategy is unnecessarily
conservative in practice.  To obtain a better empirical performance,
instead, we use
\cref{alg:normcs}---an anytime-valid technique--to ensure early-termination whenever possible.
In lieu
of proving termination, we backstop
\cref{alg:normcs} with \cref{thm:sample},
which leads to at most doubling the number of samples required.

\newcommand{\betmartup}{\texttt{\small BettingMartingale.Update}\xspace}
\begin{algorithm}[H]
\caption{Normalization CS to compute $\beta_t$. The subroutine \betmartup is defined in \cref{app:explainnormcs}.}
\label{alg:normcs}
	\renewcommand{\algorithmicrequire}{\textbf{Input:}}
	\renewcommand{\algorithmicensure}{\textbf{Output:}}
	\newcommand{\algorithmicbreak}{\textbf{break}}
    \newcommand{\BREAK}{\STATE \algorithmicbreak}
	\begin{algorithmic}[1]
		\REQUIRE $\widehat f_t$ (from regression oracle \sqalgtext);
		exploration parameter $\gamma > 0$;
		failure probability $\delta$;
		and $\kappa_{\infty} \geq 1$. \\
		\hfill \algcommentlight{It suffices to take $\kappa_\infty = 24$}
        \STATE Let $n_{\max} = O\left(\tau \log (\nicefrac{\gamma}{\delta})\right)$ \hfill\algcommentlight{from \cref{thm:sample}}
        \STATE $l, u \leftarrow \frac{1 - \tau}{\gamma}, 1$ \hfill \algcommentlight{Because $\widehat f_t(x_t, \cdot) \in [0, 1]$}
		\FOR{$n = 1, 2, \dots, n_{\max}$}
        \STATE Sample $a_n \sim \mu$. \alglinelabel{alg:normcs:sample}
        \STATE Let $g_n(\cdot) = \gfunctwoarg{(\cdot)}{a_n}$.
        \STATE $l_n, u_n \leftarrow \betmartup(g_n; \kappa_\infty; \delta)$
        \IF{$l_n > u_n$}
          \STATE \textbf{return} $l_n$ \hfill \algcommentlight{Satisfies \cref{eqn:findbeta} w.p. $(1 - \delta)$} \alglinelabel{alg:normcs:stoppingrule}
        \ENDIF
		\ENDFOR
    \STATE \textbf{tail call} \cref{thm:sample} \hfill \algcommentlight{Never happens in practice}
	\end{algorithmic}
\end{algorithm}

\begin{restatable}{theorem}{thmNormCS}
\label{thm:normcs}
If \cref{alg:normcs} returns a value on line \ref{alg:normcs:stoppingrule},
that value satisfies \cref{eqn:findbeta} with
probability at least $(1 - \delta)$
with respect to the realizations from line \ref{alg:normcs:sample}.
\end{restatable}
\begin{proof} See \cref{app:proofthmnormcs}
\end{proof}

In practice, \cref{alg:normcs} is vastly more sample
efficient than the procedure from
\cref{thm:sample}: see \cref{tab:normcsvsfixedbern} for an empirical comparison.
\cref{alg:normcs} operates by maintaining two betting
martingales, one of which tries to refine
a lower bound on $\beta$ and the other an
upper bound.  We defer complete details to
\cref{app:explainnormcs}.

As a motivation, note the combination of
betting martingales and no-regret algorithms
yields a test with asymptotic optimal power~\citep{casgrain2022anytime},
but which can be safely composed with any
stopping rule (e.g., line \ref{alg:normcs:stoppingrule} of \cref{alg:normcs}).
Early stopping is advantageous to the extent
$\widehat f_t$ is closer to a constant function,
because evidence regarding the normalization
constant accumulates more rapidly than accounted
for by \cref{thm:sample}.

\begin{algorithm}[H]
	\caption{Sampling routine}
	\label{alg:highdsampler}
	\renewcommand{\algorithmicrequire}{\textbf{Input:}}
	\renewcommand{\algorithmicensure}{\textbf{Output:}}
	\newcommand{\algorithmicbreak}{\textbf{break}}
    \newcommand{\BREAK}{\STATE \algorithmicbreak}
	\begin{algorithmic}[1]
        \REQUIRE $\widehat f_t$ (from regression oracle \sqalgtext);
		$\beta_t$ (from \cref{alg:normcs});
        exploration parameter $\gamma > 0$
		\WHILE{true}
		\STATE Sample $a_t \sim \mu$.
		\STATE Compute $p_{\text{accept}} \ldef \frac{1}{\gfuncdenom{\beta_t}}$.
		\alglinelabel{alg:highdsampler:paccept}
		\STATE With probability $p_{\text{accept}}$, return $a_t$.
		\alglinelabel{alg:highdsampler:return}
		\ENDWHILE
	\end{algorithmic}
\end{algorithm}

\paragraph{Efficiently sampling $a_t \sim P_t$} \cref{alg:highdsampler} is an efficient rejection sampling on the density from \cref{eqn:cappeddensity}.  Note that $p_{\text{accept}}$ on line \ref{alg:highdsampler:paccept} is proportional
to the desired sampling density $P_t$ defined in \cref{eqn:cappeddensity}, but at most 1.
Hence we have the following two established properties
of rejection sampling:
\begin{enumerate}
	\item
If
line \ref{alg:highdsampler:return}
returns an action $a_t$, then the action $a_t$ is distributed
according to \cref{eqn:cappeddensity};
\item
The number of samples required before \cref{alg:highdsampler} terminates is
geometrically distributed with mean
$\kappa_t \tau$.  In particular, with
high probability the number of samples is
$O(\kappa_t \tau)$ due to exponential tail bounds.
\end{enumerate}

\paragraph{Computing $\kappa_t$}
The astute reader will notice that $\kappa_t$
need not be computed explicitly for
\cref{alg:capped}, i.e., for online inference.
However an estimate of $\kappa_t$ might
be useful for having more accurate
importance-weights for offline reuse.
For our experiments we use the naive
constant estimate $\wh \kappa_t = 1$,
and leave this an area for future work.

\section{Experiments}
\label{sec:experiment}
We conduct multiple experiments in this section. In \cref{sec:normalization}, we empirically compare the performance of \cref{thm:sample} and \cref{alg:normcs}.
We compare our algorithm \cappedIGW with the previous state-of-the-art algorithm \smthigw \citep{zhu2022contextual} in terms of both the online performance (\cref{sec:online}) and the offline utility (\cref{sec:offline}).
We also demonstrate why \smthigw lacks offline utility in \cref{sec:smthigw_greedy}.
Code to reproduce all experiments available
at \url{https://github.com/mrucker/onoff_experiments}.

\subsection{Normalization CS}
\label{sec:normalization}

This experiment establishes the empirical
validity and efficacy of \cref{alg:normcs}.
For these simulations we use the unit interval
as the action space; Lebesgue reference measure;
$\widehat f_t(x_t, a_t) = 1_{2 a_t \tau > 1}$,
corresponding to loss function which is a narrow ``needle in the haystack'';
failure probability $\delta = 2.5\%$;
and $\kappa_\infty = 24$.  As indicated
in \cref{tab:normcsvsfixedbern},
\cref{alg:normcs} is a vast improvement over
the procedure from \cref{thm:sample}.  Note
in \cref{tab:normcsvsfixedbern}, $\kappa_t$
is the true value computed analytically
from the $\beta_t$ produced by \cref{alg:normcs}.

\begin{table}[h]
\caption{\cref{alg:normcs} is vastly more sample efficient than the procedure from \cref{thm:sample}.  The $n$ and $\kappa_t$ from \cref{alg:normcs} are random variables: shown are 95\% bootstrap CI of the realization (\emph{not} the population mean) over different sampler seeds.}
\vskip 0.15in
\begin{center}
\begin{small}
\begin{sc}
\begin{tabular}{ccccc}
\toprule
$\tau$ &
$\gamma$ &
$n$ (\hyperref[thm:sample]{Thm~\ref*{thm:sample}}) &
$n$ (\hyperref[alg:normcs]{Alg~\ref*{alg:normcs}}) &
$\kappa_t$ (\hyperref[alg:normcs]{Alg~\ref*{alg:normcs}}) \\
\midrule
2 & 16 & 942 & [18, 24] & [1.3, 3.0] \\
20 & 304 & 13496 & [123, 227] & [10.2, 11.8] \\
200 & 6368 & 177141 & [2254, 2788] & [1.8, 23.6] \\
\bottomrule
\end{tabular}
\end{sc}
\end{small}
\end{center}
\vskip -0.1in
\label{tab:normcsvsfixedbern}
\end{table}

\subsection{Online Regret}
\label{sec:online}

   \begin{figure}[t]
   	\centering
   	\includegraphics[width=0.49\textwidth]{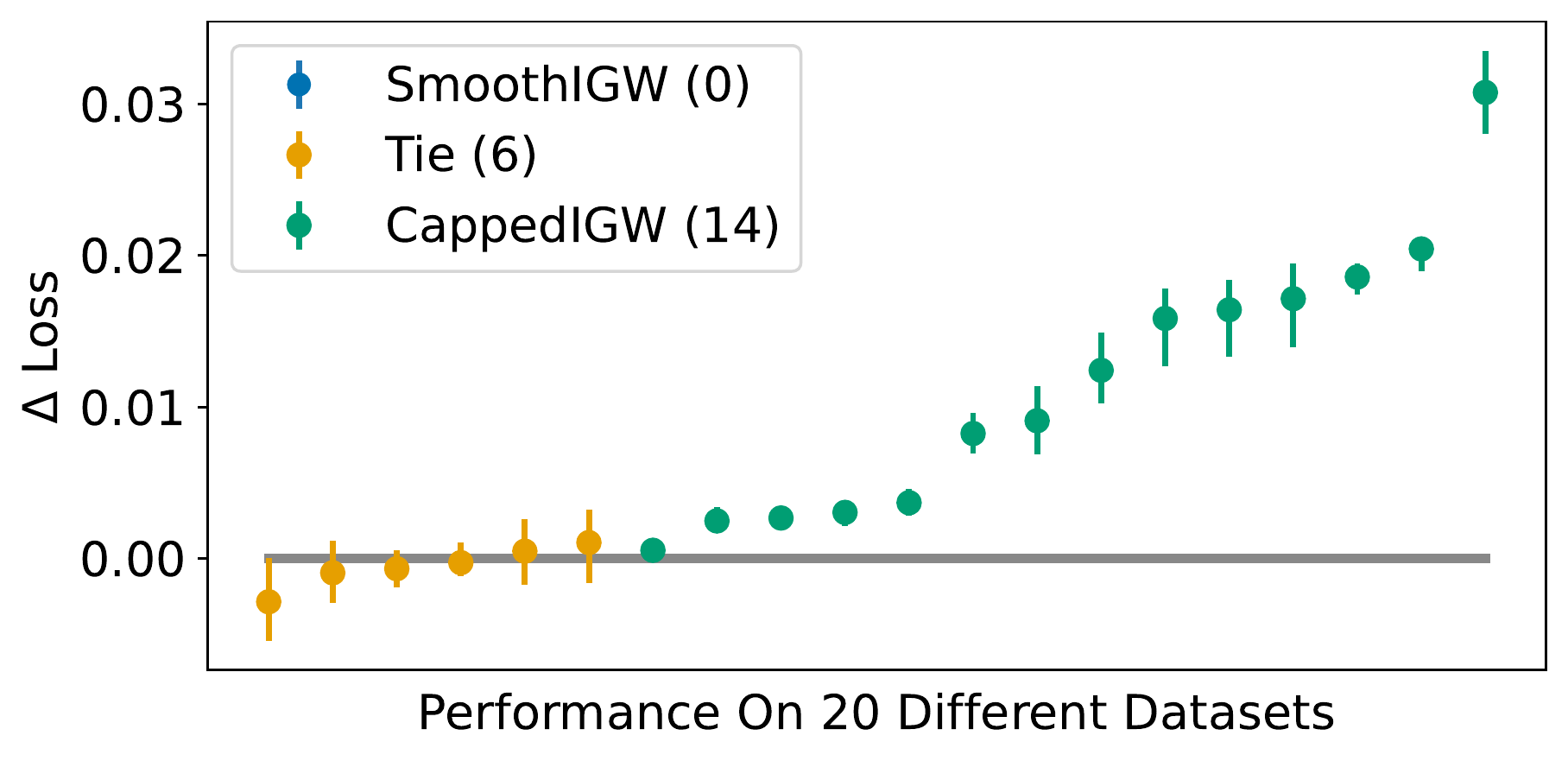}
         \vskip -12pt
   	\caption{Consistent with their similar theoretical guarantees, online performance
        of {\scriptsize \smthigw} and {\scriptsize \cappedIGW} is similar, although {\scriptsize \cappedIGW} enjoys a slight advantage. Each datapoint represents a single dataset.  Plotted here is $\left(\text{Loss}({\scriptsize \smthigw})-\text{Loss}({\scriptsize \cappedIGW})\right)$ with 90\% bootstrap CIs, i.e., larger values in the plot indicate ${\scriptsize \cappedIGW}$ is outperforming ${\scriptsize \smthigw}$.  Win/tie/loss is determined by if the CI contains 0.}
        \vskip -12pt
   	\label{fig:online_diff}
   \end{figure}

Here we demonstrate that \smthigw has similar
online regret to \cappedIGW. We use twenty
regression datasets converted to contextual bandit
datasets with action space $\mathcal{A} = [0, 1]$ via a
supervised-to-bandit transformation
\cite{bietti2021contextual}. Each dataset
is individually shifted and scaled so that
target value $y_t \in [0, 1]$.  When an
algorithm plays action $a_t \in [0, 1]$,
it receives bandit feedback $\ell_t(a_t) \ldef | y_t - a_t |$.

We assess each algorithm (\smthigw, \cappedIGW)
on progressive validation
loss.~\cite{blum1999beating}. For each
dataset we run both algorithms using the same
set of $30$ different
seeds, where a seed controls
all non-determinism (including data set
shuffling, parameter initialization, and
action sampling). For each dataset we compute
the average of the paired (by seed) differences
between each algorithm, and then
compute a 90\% bootstrap confidence interval.

The two algorithms are declared to have tied
on a dataset when the 90\% CI
for their difference contains a $0$. Otherwise
one of the algorithms is declared to win. In
total we observe five ties, one small \smthigw
win, and fourteen small \cappedIGW wins. The
complete result can be seen in Figure~\ref{fig:online_diff}.

This experiment also demonstrates the effectiveness 
of \cref{alg:normcs} within \cappedIGW. In this 
experiment \cappedIGW determines $\beta_t$ each 
iteration using \cref{alg:normcs} with $\kappa_{\infty} = 4$.

For further details (e.g., model class for $\wh f_t$) see \cref{app:onlineregretexperiment}.

\subsection{Offline Utility}
\label{sec:offline}

This experiment provides an example of
the increased utility of \cappedIGW's data exhaust
for offline learning relative to \smthigw's exhaust.  
Here we mimic a typical production goal of evaluating
a more complicated model class than was
used online qua \citet[$\S4.6$]{gomez2015netflix}.
As shown in \cref{fig:offline_diff},
offline learners trained on \cappedIGW exhaust exhibit
statistically significant smaller average loss on twelve of
twenty datasets.

\begin{figure}[t]
\begin{minipage}[t]{0.48\linewidth}
\vskip 0pt
{
    \centering
    \includegraphics[width=0.99\textwidth]{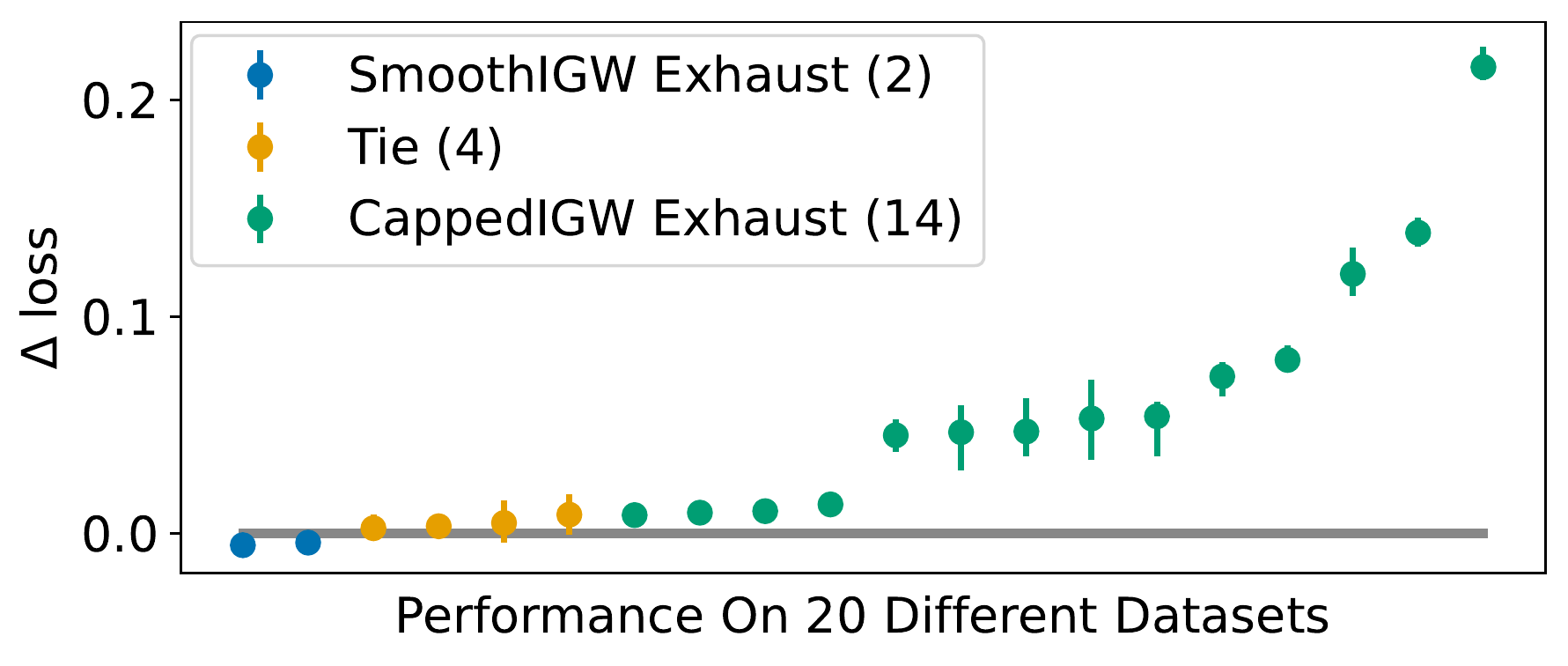}
}
\end{minipage}
\hfill
\begin{minipage}[t]{0.48\linewidth}
\vskip 0pt
{
    \centering
    \caption{Policies trained offline using {\scriptsize \cappedIGW} exhaust exhibit less average loss online compared to policies trained offline using {\scriptsize \smthigw} exhaust. Plotted here is $\left(\text{Loss}({\scriptsize \smthigw})-\text{Loss}({\scriptsize \cappedIGW})\right)$ with 90\% bootstrap CIs, i.e., larger values in the plot indicate training on ${\scriptsize \cappedIGW}$ exhaust is superior to training on ${\scriptsize \smthigw}$ exhaust. Win/tie is determined by the CI containing $0$.}\label{fig:offline_diff}

}
\end{minipage}
\end{figure}

\begin{figure}[t]
\begin{minipage}[t]{0.48\linewidth}
\vskip 0pt
{
    \vspace{0.23cm}
    \centering
    \caption{Over time {\scriptsize \smthigw} increasingly plays its greedy action, which does not have a well defined importance-weight. (Not shown) With {\scriptsize \cappedIGW} the data exhaust always has an importance-weight.}
    \label{fig:fractionofindic}
}
\end{minipage}
\hfill
\begin{minipage}[t]{0.48\linewidth}
\vskip 0pt
{
    \centering
    \includegraphics[width=0.99\textwidth]{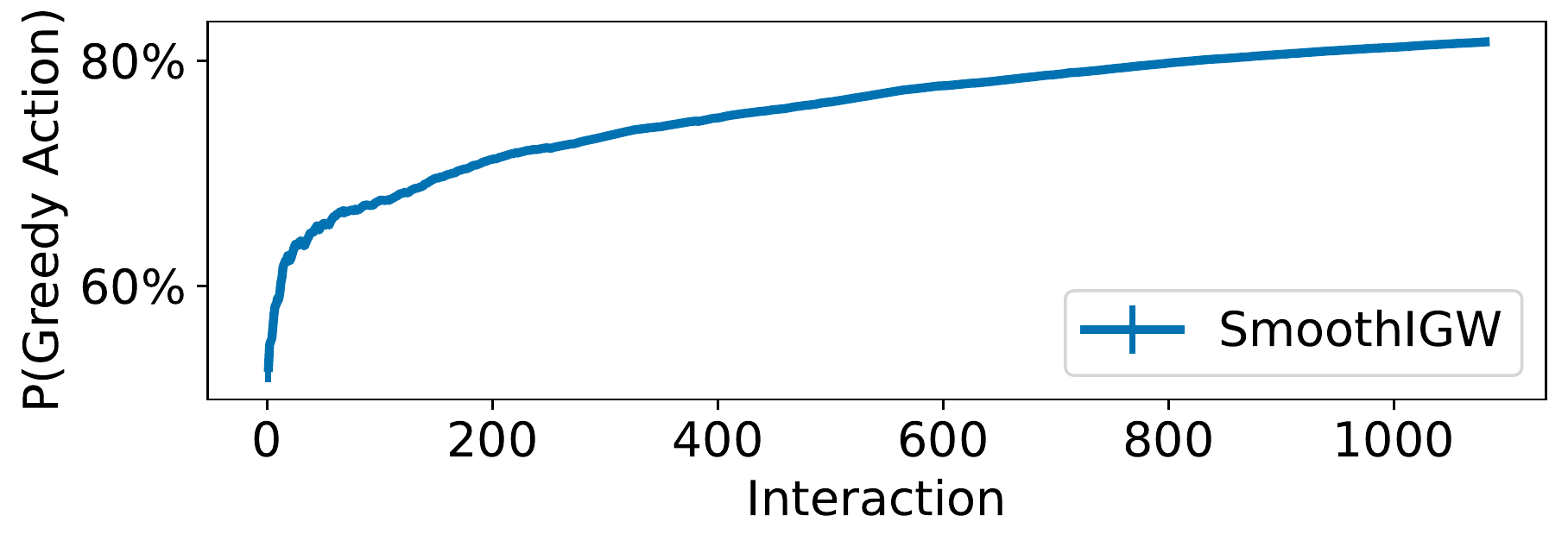}
    \vspace{-.5cm}
}
\end{minipage}
\end{figure}

To generate data exhaust all $(x_t, a_t, \wh{P}_t(a_t), \ell_t(a_t))$ were logged during the online experiments described in Section~\ref{sec:online}, where
$$
\wh{P}_t(a_t) \ldef \gfunc{\beta_t},
$$
i.e., we (naively) estimate
$\wh{\kappa_t} = 1$.  We use the inverse
of $\wh{P}_t(a_t)$ as the importance weight.

For each resulting
dataset (\smthigw exhaust or
\cappedIGW exhaust), the best of two
off-policy learning methods was selected:
the direct method~\cite{dudik2011doubly},
which does not use importance-weights; and 
clipped IPS~\cite{strehl2010learning}, where for
\smthigw exhaust we assign the greedy action
the maximum importance weight of 5.

To train the offline models data exhaust is split 80\%-10\%-10\% for training, validation and testing respectively. Training epochs are performed on the training set until a decrease in model performance is observed on the validation set. After training learners are assessed using the average loss on
the test set. Note validation and test evaluation are independent of what data exhaust was used to train, as the source datasets contain the true label and therefore admit on-policy evaluation.

For each dataset we run the offline learners 30 times using the exhaust files generated from the 30 online seeds. For each dataset we compute the average of the paired (by exhaust) differences between offline learners and then compute a 90\% bootstrap confidence interval.

For further details see \cref{app:offlineregretexperiment}.

\subsection{\smthigw Increasingly Plays Greedy}
\label{sec:smthigw_greedy}

Here we show the frequency that \smthigw plays its greedy action during the online experiment described in \cref{sec:online}. This is not a problem for online performance. Rather, as described in Section~\ref{sec:statistical_background}, this only becomes a problem when attempting to conduct post-hoc analysis with importance-weighting techniques. We can see in \cref{fig:fractionofindic} that by the $1,000^{th}$ learning iteration in the online experiment over 80\% of played actions no longer have usable importance weights for post-hoc analysis.

\subsection{\cappedIGW Sensitivity to \texorpdfstring{$\kappa_\infty$}{κ-∞}}

Here we look at the effect of varying levels of $\kappa_\infty$ on online and offline performance using our 20 Datasets. For these experiments we set $\kappa_\infty$ equal to 2, 4, and 24 (note, $\kappa_\infty$ was 4 for  experiments in Section~\ref{sec:online} and \ref{sec:offline}).

In our experiments the value of $\kappa_\infty$ strongly impacted the number of samples required to estimate $\beta_t$ with smaller values of $\kappa_t$ requiring more samples (\cref{fig:k_samps}). This is expected given that smaller values of $\kappa_\infty$ indicate tighter confidence bounds on $\beta_t$.

\begin{figure}[t]
\begin{minipage}[t]{0.49\linewidth}
\vskip 0pt
{
    \centering
    \includegraphics[width=0.99\textwidth]{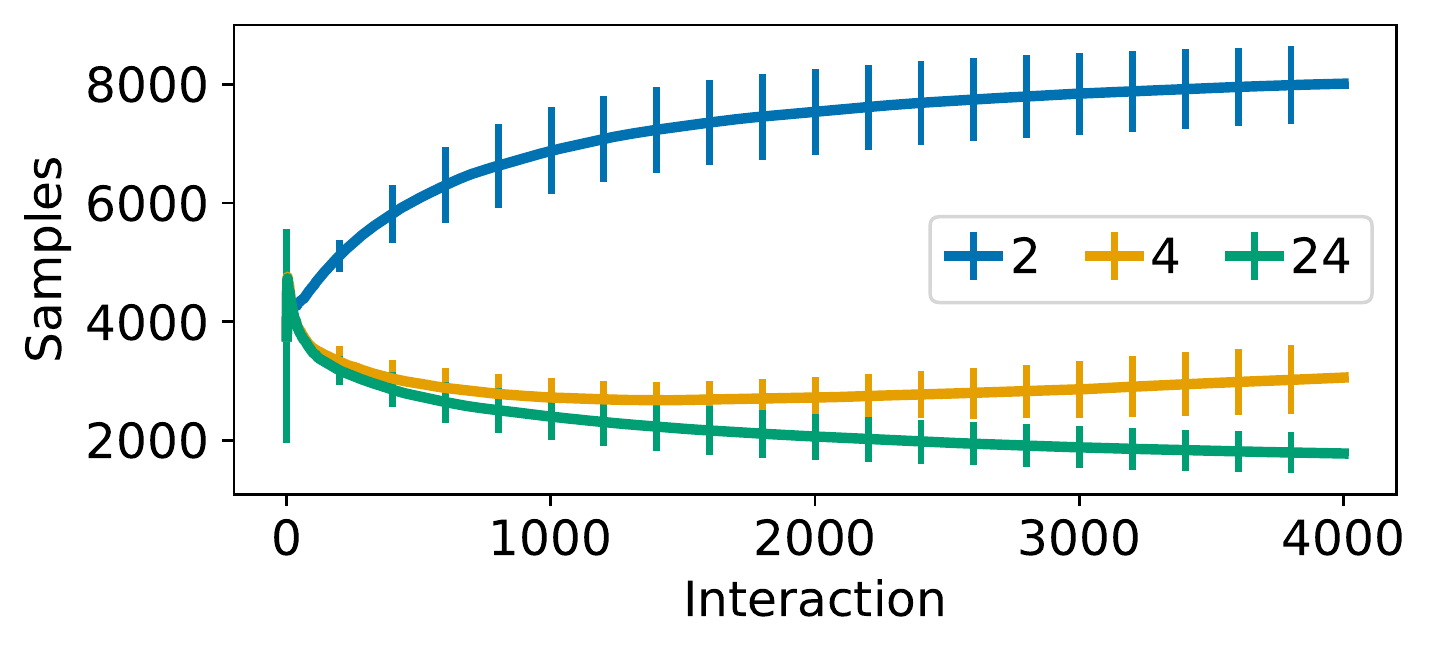}
}
\end{minipage}
\hfill
\begin{minipage}[t]{0.48\linewidth}
\vskip 0pt
{
    \centering
    \caption{We see an increase in the number of samples required to estimate $\beta_t$ on each learning update of {\scriptsize \cappedIGW} as $\kappa_\infty$ shrinks.} \label{fig:k_samps}
}
\end{minipage}
\end{figure}

\begin{figure}[t]
\begin{minipage}[t]{0.48\linewidth}
\vskip 0pt
{
    \centering
    \includegraphics[width=0.99\textwidth]{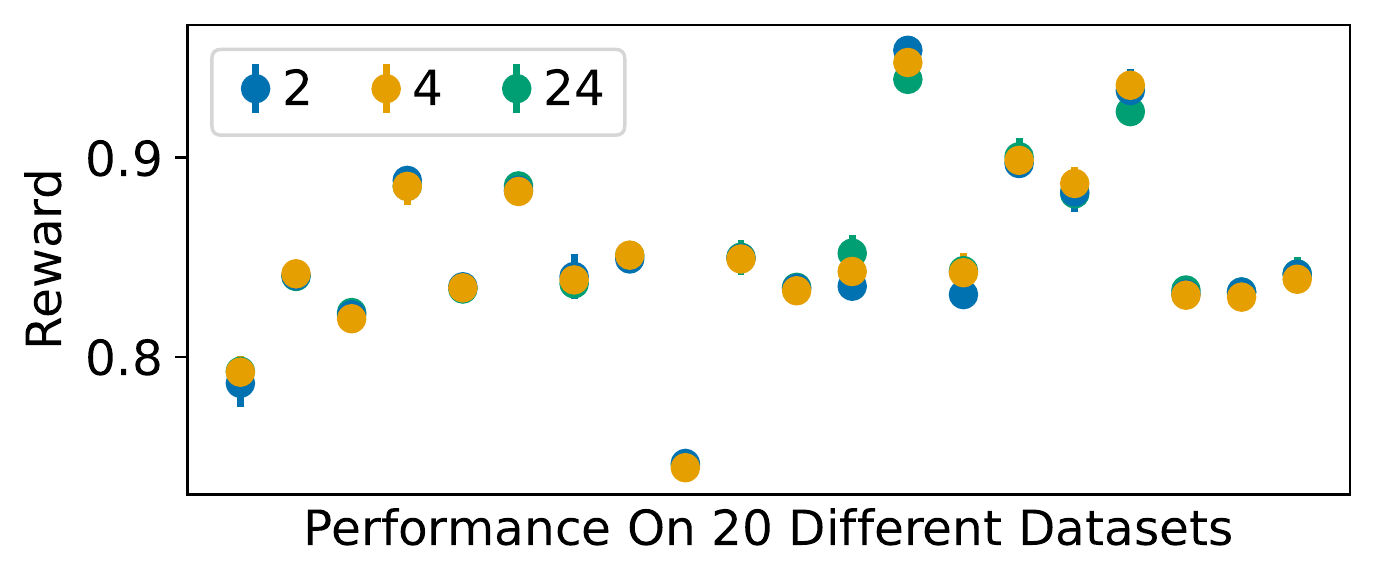}
    \caption{We see very little difference in online performance across our 20 datasets for varying levels of $\kappa_\infty$ in {\scriptsize \cappedIGW}.} \label{fig:k_online}
}
\end{minipage}
\hfill
\begin{minipage}[t]{0.48\linewidth}
\vskip 0pt
{
    \centering
    \includegraphics[width=0.99\textwidth]{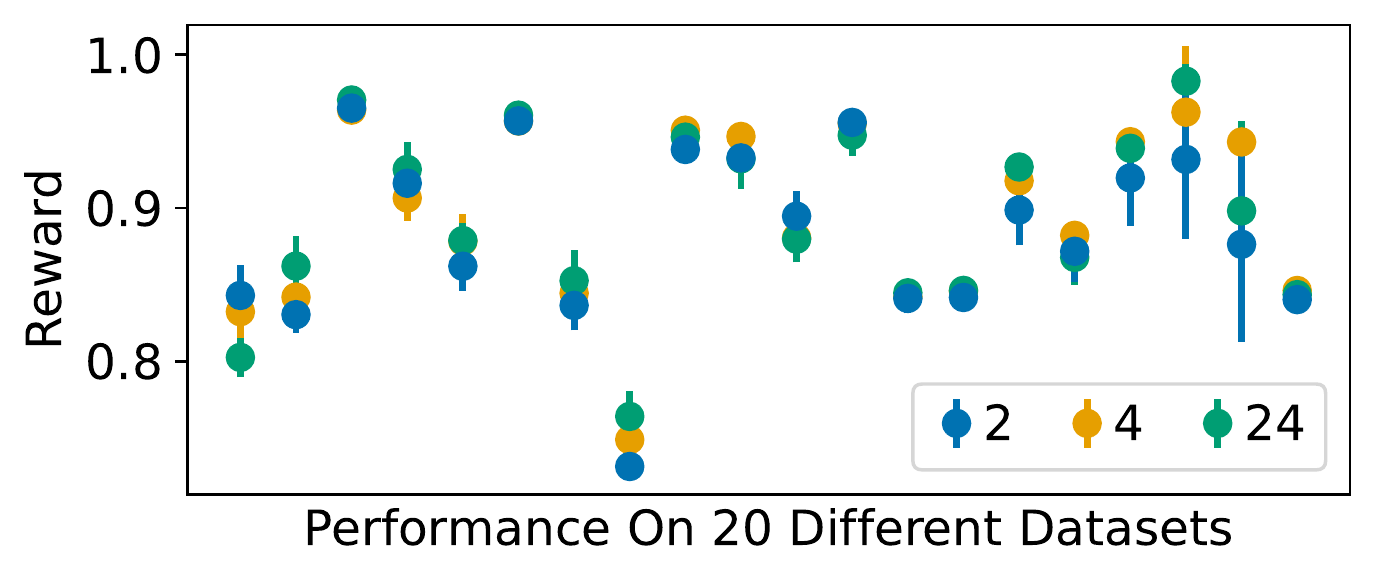}
    \caption{We tend to see increased offline performance across our 20 datasets with data exhaust from {\scriptsize \cappedIGW} when $\kappa_\infty$ is 4 or 24.}
    \label{fig:k_offline}
}
\end{minipage}
\end{figure}

We observe a negligible impact in online performance for the three levels of $\kappa_\infty$ (\cref{fig:k_online}). In most datasets the average reward seen was nearly identical at all levels. At the same time we observe a slight increase in offline utility when training on the online exhaust generated with $\kappa_\infty$ equal to 4 or 24 (\cref{fig:k_offline}).

For further details see \cref{app:sensitivityexperiment}.

\section{Additional Related Work}

In this section, we briefly highlight related
work that we have not already mentioned in previous sections.

\paragraph{Large action spaces with additional assumptions}
Unlike contextual bandits with finite action sets, infinite
(or very large) action space contextual
bandits are minimax
intractable---as observed from the lower bound in \cite{agarwal2012contextual}.
Nonetheless, the setting with infinite action spaces are highly
practical in many real-world scenarios, e.g., in large-scale recommender systems.
To make progress in this setting, researchers have develop algorithms that work with additional modeling assumptions, such as
contextual bandits with linear functions \citep{auer2002using, li2010contextual, abbasi2011improved}, with linearly-structured actions and general (context) function approximation \citep{foster2020adapting, xu2020upper, zhu2022large}, with
Lipschitz/H\"older regression functions \citep{kleinberg2004nearly, hadiji2019polynomial}, and with convex functions \citep{lattimore2020improved}.
While these modeling assumptions have lead to fruitful theoretical guarantees, they might be violated in practice.

\paragraph{Large action spaces with smooth regret}
An alternative line of research to tackling the large
action spaces problems, in which
this paper sits, is to weaker the competing benchmark
to avoid the otherwise minimax negative result.
This idea was first proposed in non-contextual bandits by \citet{chaudhuri2018quantile}, where they compete against the $1-\alpha$th quantile (of reward) instead of highest reward.
In the case with contextual bandits, \citet{krishnamurthy2020contextual}
proposed a variant of the smooth regret (defined in \cref{sec:smooth_regret})
for agnostic policy-based analysis.
\citet{krishnamurthy2020contextual} develops algorithms that are statistically optimal, but computationally intractable;
a computationally tractable instantiation was
later developed in \citet{majzoubi2020efficient} (but with slightly weaker statistical performance).
We remark here that, even though our definition of smoothed regret (in \cref{sec:smooth_regret})
dominates the one appearing in \citet{krishnamurthy2020contextual}, our
approach requires an additional realizability
assumption to reduce to regression (instead of classification);
\citet{foster2020adapting} shows how to
manage misspecification within a reduction
to regression framework.

\paragraph{Offline learning in contextual bandits}
Offline learning, or off-policy evaluation, considers the problem of learning 
a new policy/model only using historic logging data collected from other online policies.
Because offline learning permits learning/testing without
costly online exploration, it has been used in many real-world applications, 
such as recommender systems \citep{thomas2017predictive} and healthcare industry \citep{nie2021learning}.
Focusing on contextual bandits, the method of inverse propensity scoring (IPS) \citep{horvitz1952generalization} 
has been extensively used to correct the mismatch between action distributions under 
the offline and online policies. 
Besides the IPS method, 
the direct method (DM) \citep{dudik2011doubly, rothe2016value} has also been used in offline learning where 
the learner first learns a reward estimator based on the offline data and then evaluates the new policies.

\section{Discussion}

This work exhibits a statistical free lunch:
the online regret guarantee of an algorithm is
essentially unchanged, while the subsequent
offline utility of the data exhaust is
increased.\footnote{Although there is a computational cost, 
this is arguably mitigated by eliminating the $\argmin$ oracle.}  
We speculate this is not typical
but rather an artifact of the sub-optimality
of the prior technique.  In other words, we
anticipate that online regret and offline
utility are conflicting objectives that must
be traded off, suggesting a currently unknown
Pareto frontier remains to be discovered.
The empirical study of \citet{williams2021challenges} provides
evidence in this direction.

\bibliography{main}
\bibliographystyle{styles/icml2023/icml2023}

\newpage
\appendix
\onecolumn

\section{Background: Minmax Reduction Design \label{sec:background}}
Our approach is based on the work of \citet{foster2021statistical}, which we review here. From this work, we define the Decision-Estimation Coefficient for any smoothness level $\tau \geq 1$, function class $\cF$, context $x \in \cX$, and $\AlgSq$ estimate $\wh f$:
\begin{equation}
    \dec_\gamma(\cF; \wh f, x) \ldef \inf_{P \in \Delta(\cA)} \sup_{Q \in \cQ_\tau} \sup_{f \in \cF} \E_{a\sim P, a^\star \sim Q} \brk*{f(x, a) - f(x, a^\star) - \frac{\gamma}{4} \cdot \prn*{\widehat f(x,a ) - f(x, a)}^2},
    \label{eqn:minimax_game}
\end{equation}
where $f$ is the true loss function, $Q$ is the optimal smoothed policy with respect to $f$, $P$ is a policy of our choosing, and $\gamma$ is a tunable learning rate. Note that with this formulation $\smtht(x) = -\sup_{Q \in \cQ_\tau}\E_{a \sim Q}\brk{-f^\star(x,a)}$ with respect to results in the paper.

Our goal is to construct $P$ such that we can derive an upper bound on $\dec$. Because $\dec$ is the difference between the expectation of $\regcbt$ and $\regsq$ an upper bound on $\dec$ implies that $\regcbt(T)$ has an upper bound in terms of $\regsq(T)$. This allows us to reduce the CB problem to simply minimizing the $\regsq(T)$ via any regression oracle $\AlgSq$ of our choosing.

For our work we prove a bound on the Decision-Estimation Coefficient in \cref{app:decbound} and from there derive a regret bound in \cref{app:regbound}.

\section{Proof of Theorem~\ref{thm:regret}}
\label{app:proofthmregret}

\thmRegret*

The proof proceeds by first bounding the \dectext~\cite{foster2021statistical},
after which the regret bound follows almost directly.

\subsection{Bounding the \dectext \label{app:decbound}}

With respect to any context $x\in \cX$ and estimator $\wh
f $ obtained from \sqalgtext, we consider \cref{eqn:minimax_game}.

\begin{lemma}[\citet{zhu2022contextual}]
    \label{lm:dec_chi2}
    Fix constant $\gamma > 0$ and context $x\in\cX$ . For any measures $P$ and $Q$ such that $Q \ll P$, we have
    \begin{align*}
        &  \sup_{f \in \cF} \E_{a\sim P, a^\star \sim Q} \brk*{f(x, a) - f(x, a^\star) - \frac{\gamma}{4} \cdot \prn*{\widehat f(x,a ) - f(x, a)}^2} \\
        & \leq \E_{a \sim P} \brk[\big]{\widehat f(x, a)} - \E_{a \sim Q} \brk[\big]{\widehat f(x, a)} +\frac{1}{\gamma} \cdot  \E_{a \sim P}\brk*{ \prn*{\frac{dQ}{dP}(a) - 1}^2 }.
    \end{align*}
\end{lemma}

Subsequently we omit the dependence on the context $x\in \cX$, and use abbreviations $f(a) \ldef f(x,a)$ and $\wh f(a) \ldef \wh f(x,a)$.

We first notice that for any $Q \in \cQ_\tau$ we have $Q \ll P_t$ for $P_t$
defined in \cref{eqn:cappeddensity}: since (i) $Q \ll \mu$ by definition,
and (ii) $\mu \ll P_t$.  Therefore, applying \cref{lm:dec_chi2} we have
\begin{align}
        & \E_{a\sim P_t, a^\star \sim Q} \brk*{f( a) - f(a^\star) - \frac{\gamma}{4} \cdot \prn*{\widehat f(a ) - f(a)}^2}\nonumber\\
        & \leq
\E_{a\sim P_t} \brk[\big]{\widehat f(a)} - \E_{a \sim Q} \brk[\big]{\widehat f(a)} + \frac{1}{\gamma} \cdot  \E_{a \sim P_t}\brk*{ \prn*{\frac{dQ}{dP_t}(a) - 1}^2 }. \nonumber
\end{align}
Denote $p(a) = \frac{dP_t}{d\mu}(a)$ and $q(a) = \frac{dQ}{d\mu}(a)$.  Continuing
    \begin{align}
        & \E_{a\sim P_t} \brk[\big]{\widehat f(a)} - \E_{a \sim Q} \brk[\big]{\widehat f(a)} + \frac{1}{\gamma} \cdot  \E_{a \sim P_t}\brk*{ \prn*{\frac{dQ}{dP_t}(a) - 1}^2 }  \nonumber \\
        & = \E_{a \sim \mu} \brk*{p(a) \cdot \prn[\Big]{\wh f(a) - \beta} } - \E_{a \sim \mu} \brk*{q(a) \cdot \prn[\Big]{\wh f(a) - \beta } } + \frac{1}{\gamma} \cdot \E_{a \sim \mu} \brk*{p(a) \cdot \prn*{\frac{q(a)}{p(a)} - 1}^2 } \nonumber \\
        & = \E_{a \sim \mu} \brk*{p(a) \cdot \prn[\Big]{\wh f(a) - \beta } } - \E_{a \sim \mu} \brk*{q(a) \cdot \prn[\Big]{\wh f(a) - \beta } } + \frac{1}{\gamma} \cdot \E_{a \sim \mu} \brk*{q(a) \cdot \frac{q(a)}{p(a)} - 2 q(a) + p(a)  } \nonumber \\
	& = \E_{a \sim \mu} \brk*{p(a) \cdot \prn[\Big]{\wh f(a) - \beta } } +
	\frac{1}{\gamma} \cdot \E_{a \sim Q} \brk*{\frac{q(a)}{p(a)} - \gamma \cdot \prn*{\wh f(a) - \beta } } - \frac{1}{\gamma} \nonumber
 \\
 & = \E_{a \sim \mu} \brk*{p(a) \cdot \max\crl[\Big]{0, \wh f(a) - \beta } }
 + \frac{1}{\gamma} \cdot \E_{a \sim Q} \brk*{\frac{q(a)}{p(a)} - \gamma \cdot \max\crl*{0, \wh f(a) - \beta } }  \nonumber \\
 & \quad + \E_{a \sim \mu} \brk*{\left( p(a) - q(a) \right) \cdot \min \crl[\big]{0, \wh f(a) - \beta } } - \frac{1}{\gamma}. \label{eqn:decpoint1}
   \end{align}
Now we note the definition of $P_t$ implies
\begin{align}
\E_{a \sim \mu} \brk*{p(a) \cdot \max\crl[\Big]{0, \wh f(a) - \beta } } &\leq \frac{\kappa_t \tau}{\gamma}; \label{eqn:miniboundone}
\end{align}
furthermore, the constraints $q(a) \leq \tau$ and $\kappa_t \geq 1$ imply
\begin{align}
&\frac{1}{\gamma} \cdot \E_{a \sim Q} \brk*{\frac{q(a)}{p(a)} - \gamma \cdot \max\crl*{0, \wh f(a) - \beta } } \nonumber \\
&\leq \frac{1}{\gamma} \E_{a \sim Q} \brk*{\frac{1 + \gamma \cdot \max\crl*{0, \wh f(a) - \beta }}{\kappa_t} - \gamma \cdot \max\crl*{0, \wh f(a) - \beta } } \nonumber \\
&\leq \frac{1}{\gamma}; \label{eqn:miniboundtwo}
\end{align}
and finally
\begin{align}
&\E_{a \sim \mu} \brk*{\left( p(a) - q(a) \right) \cdot \min \crl[\big]{0, \wh f(a) - \beta } } \nonumber \\
&\leq (\kappa_t - 1) \tau \ \E_{a \sim \mu} \brk*{\min \crl[\big]{0, \wh f(a) - \beta } } \nonumber \\
&\leq 0. \label{eqn:miniboundthree}
\end{align}
Substituting \cref{eqn:miniboundone}, \cref{eqn:miniboundtwo}, and \cref{eqn:miniboundthree} into \cref{eqn:decpoint1} yields
\begin{align}
        \E_{a\sim P_t, a^\star \sim Q} \brk*{f( a) - f(a^\star) - \frac{\gamma}{4} \cdot \prn*{\widehat f(a ) - f(a)}^2} & \leq \frac{\kappa_t \tau}{\gamma}. \label{eqn:finaldecbound}
\end{align}

\subsection{Finishing the proof \label{app:regbound}}

\DeclarePairedDelimiter{\sq}{[}{]}

This part is almost verbatim from \citet[Thm 1, Appendix A.2]{zhu2022contextual}, but included for completeness.

   We use abbreviation $f_t(a) \ldef f(x_t,a)$ for any $f \in \cF$.
    Let $a^\star_t$ denote the action sampled according to the best smoothing kernel within $\cQ_\tau$ (which could change from round to round).
    We let $\cE$ denote the good event where the regret guarantee stated in \cref{assumption:regression_oracle} (i.e., $\regsq(T) \ldef \regsq(T, T^{-1})$) holds with probability at least  $1- T^{-1}$. Conditioned on this good event, following the analysis provided in \citet{foster2020adapting}, we decompose the contextual bandit regret as follows.
    \begin{align*}
	    \E \brk*{\sum_{t=1}^T f_t^\star(a_t) - f_t^\star(a^\star_t)}
        & = \E \sq*{\sum_{t=1}^T f_t^\star(a_t) - f_t^\star(a^\star_t) - \frac{\gamma}{4} \cdot \prn*{\wh f_t( a_t) - f_t^\star(a_t)}^2}
        + \frac{\gamma}{4} \cdot  \E \sq*{\sum_{t=1}^T \prn*{\wh f_t(a_t) - f_t^\star(a_t)}^2} \nonumber \\
        & \leq T \cdot \frac{\kappa_\infty \tau}{\gamma} + \frac{\gamma}{4} \cdot  \E \sq*{\sum_{t=1}^T \prn*{\wh f_t(a_t) - f_t^\star(a_t)}^2},
    \end{align*}
    where the bound on the first term follows from \cref{eqn:finaldecbound}. We analyze the second term below.
    \begin{align*}
        & \frac{\gamma}{4} \cdot \E \sq*{\sum_{t=1}^T \prn*{ \prn*{\wh f_t(a_t) - \ell_t(a_t)}^2 -\prn[\Big]{f^{\star}(a_t) - \ell_t(a_t)}^2  + 2 \prn[\Big]{\ell_t(a_t) - f^\star_t(a_t)} \cdot \prn[\Big]{\wh f_t(a_t) - f^\star_t(a_t)}}} \\
        & =  \frac{\gamma}{4} \cdot \E \sq*{\sum_{t=1}^T \prn*{ \prn*{\wh f_t(a_t) - \ell_t(a_t)}^2 -\prn[\Big]{f^\star_t(a_t) - \ell_t(a_t)}^2  }} \\
        & \leq \frac{\gamma}{4} \cdot \regsq(T),
    \end{align*}
    where on the second line follows from the fact that $\E \sq{\ell_t(a) \mid x_t} = f^\star(x_t,a)$ and $\ell_t$ is conditionally independent of $a_t$, and the third line follows from the bound on regression oracle stated in \cref{assumption:regression_oracle}.
    As a result, we have
    \begin{align*}
        \regcbt(T) \leq \frac{T \kappa_\infty \tau}{\gamma} + \frac{\gamma}{4} \cdot \regsq(T) + O(1),
    \end{align*}
    where the additional term $O(1)$ accounts for the expected regret suffered under event  $\neg \cE$.
    Taking $\gamma = \sqrt{8 T \kappa_\infty \tau/\regsq(T)}$ leads to the desired result.

\section{Proof of Theorem~\ref{thm:sample}}
\label{app:proofthmsample}

\thmSample*

We elide the contextual dependence here, as $x_t$ is a constant for all
of these operations.

Define $$
\begin{aligned}
g(a; \beta) &\ldef \frac{\tau}{1 + \gamma \max\left(0, \widehat f_t(a) - \beta\right)}, \\
\beta_{\min} &\ldef \frac{1-\tau}{\gamma}, \\
\beta_{\max} &\ldef 1,
\end{aligned}
$$ where $\widehat f_t(a) \in [0, 1]$.  We note the following properties:
$$
\begin{aligned}
g(a; \beta) &\in [0, \tau], \\
\frac{d}{d\beta} g(a; \beta) &\in [0, \gamma g(a; \beta)], \\
\mathbb{E}_{a \sim \mu}\left[g(x, \beta_{\min})\right] &\leq 1, \\
\mathbb{E}_{a \sim \mu}\left[g(x, \beta_{\max})\right] &\geq 1, \\
\mathbb{E}_{a \sim \mu}\left[g^2(x, \beta)\right] &\leq \tau \mathbb{E}_{a \sim \mu}\left[g(x, \beta)\right].
\end{aligned}
$$

\paragraph{Fixed $\beta$ bound}
With $n$ samples we can estimate the integral $z(\beta) \doteq
\mathbb{E}_{a \sim \mu}\left[g(a, \beta)\right]$ at any fixed $\beta$
from the empirical mean $\bar{z}(\beta)$ via $$
\begin{aligned}
z(\beta) &\in \bar{z}(\beta) \pm \left( \sqrt{\frac{2 \mathbb{E}_{a \sim \mu}\left[g^2(a, \beta)\right] \ln(2/\delta)}{n}} + \tau \frac{\ln(2/\delta)}{3 n} \right) & \left(\text{Bernstein}\right) \\
&\in \bar{z}(\beta) \pm \left( \sqrt{\frac{2 \tau z(\beta) \ln(2/\delta)}{n}} + \tau \frac{\ln(2/\delta)}{3 n} \right)  & \left(\text{self-bounding}\right) \\
&\in \bar{z}(\beta) \pm \left( \frac{1}{2} z(\beta) + \frac{4 \tau \ln(2/\delta)}{n} + \tau \frac{\ln(2/\delta)}{3 n} \right),  & \left(\text{AM-GM}\right) \\
z(\beta) &\in \left[ \frac{2}{3} \left( \bar{z}(\beta) - \frac{26 \tau \ln(2/\delta)}{3 n}\right) , 2 \left( \bar{z}(\beta) + \frac{26 \tau \ln(2/\delta)}{3 n} \right) \right].
\end{aligned}
$$ with probability at least $1 - \delta$.

\paragraph{Picking the $\beta$ grid}
Suppose $$
\frac{26 \tau \ln(2/\delta)}{3 n} \leq \frac{1}{8},
$$ then $$
z(\beta) \in \left[ \frac{2}{3} \bar{z}(\beta) - \frac{1}{12}, 2 \bar{z}(\beta) + \frac{1}{4} \right],
$$
therefore $$
\begin{aligned}
\bar{z}(\beta) \in \left[ \frac{3}{16}, \frac{3}{8} \right] &\implies z(\beta) \in \left[ \frac{1}{24}, 1 \right]. \\
\end{aligned}
$$ Thus if we can evaluate $\bar{z}(\beta)$ on a grid where it increases by at most a factor of 2, then we will obtain a $\beta^*$ such that $z(\beta^*) \in \left[ \frac{1}{24}, 1 \right]$.

Using the assumptions, $$
\begin{aligned}
\bar{z}'\left(\beta\right) &\leq \gamma \bar{z}(\beta) \\
\implies \bar{z}(\beta) &\leq \bar{z}(\beta_0) \exp\left(\gamma  \left(\beta - \beta_0\right)\right)
\end{aligned}
$$
hence evaluation over a grid spaced as
$\Delta \beta = \log(2) \gamma^{-1}$ will ensure
$\bar{z}(\beta)$ does not increase by more than a
factor of 2.  Using a union bound over these
points we need $$
\begin{aligned}
\frac{\beta_{\max} - \beta_{\min}}{\Delta \beta} &\leq \frac{\tau + \gamma}{\log(2)}, \\
\frac{26 \tau \left( \log(2 \log(2)) + \log(\tau + \gamma) - \log(\delta) \right)}{3 n} &\leq \frac{1}{8}, \\
8 \frac{26 \tau \left( \log(2 \log(2)) + \log(\tau + \gamma) - \log(\delta) \right)}{3} &\leq n,
\end{aligned}
$$ thus $n = O\left(\tau \log(\nicefrac{(\tau + \gamma)}{\delta})\right)$.

\section{Explanation of Algorithm~\ref{alg:normcs}}
\label{app:explainnormcs}

Note the following discussion is localized
to a single invocation of \cref{alg:normcs},
and therefore we elide the contextual
dependence.

Using the notation from the proof of \cref{thm:sample}, note that
 $z(\beta) \ldef \mathbb{E}_{a \sim \mu}\left[g(x, \beta)\right]$
is continuous and non-decreasing in $\beta$.
Fix $\kappa_{\infty} > 1$ and define $$
\begin{aligned}
\beta_{\kappa_\infty} &\ldef \sup \left\{ \beta \left| z(\beta) \leq \kappa_{\infty}^{-1} \right. \right\}, \\
\beta_1 &\ldef \inf \left\{ \beta \left|  z(\beta) \geq 1 \right. \right\}.
\end{aligned}
$$
Given a failure probability $\delta$, we
will construct an lower confidence sequence
$L_n$ for $\beta_1$ and an upper
confidence sequence $U_n$ for $\beta_{\kappa_{\infty}}$,
each with failure probability $\delta/2$,
i.e., a pair of adapted random processes
$L_n$ and $U_n$ satisfying
\begin{align}
\mathbb{P}\left(\forall n \in \mathbb{N}: \beta_{\kappa_{\infty}} \leq U_n\right) &\geq 1 - \nicefrac{\delta}{2}, \label{eqn:uppercs} \\
\mathbb{P}\left(\forall n \in \mathbb{N}: L_n \leq \beta_1 \right) &\geq 1 - \nicefrac{\delta}{2}, \label{eqn:lowercs}
\end{align}
where our random processes are defined
on the discrete-time filtered probability space
$(\Omega, \mathcal{F}, \left\{ \mathcal{F}_n \right\}_{n \in \mathbb{N}}, \mathbb{P})$
generated by the sampling oracle in
line \ref{alg:normcs:sample}
of \cref{alg:normcs}.  Standard techniques
for achieving \cref{eqn:uppercs} and
\cref{eqn:lowercs} are described further
below: for now, assuming those properties,
note that whenever $L_n \geq U_n$, we can
conclude with probability at least
$(1 - \delta)$ that $$
\beta \in [U_n, L_n] \implies z(\beta) \in \left[\frac{1}{\kappa_{\infty}}, 1\right]
$$
which is the desired property from \cref{eqn:findbeta}.  Because $z(\beta)$
is non-decreasing and we want the smallest
$\kappa_t$ possible, we use the largest
$\beta$, and hence return $L_n$ on line \ref{alg:normcs:stoppingrule}
of \cref{alg:normcs}.

To achieve \cref{eqn:uppercs} and \cref{eqn:lowercs}, we combine betting
martingales with a no-regret algorithm,
aka ONS-m.~\citep{waudby2020estimating}
To ease exposition, we describe
the lower bound only; the upper bound
is analogous.  For the lower bound we
define the wealth process $$
\begin{aligned}
W^{(-)}_n(\nu; \beta) &= \prod_{m=1}^n \left(1 + \nu_m \left(1 - g(A_m; \beta)\right)\right),
\end{aligned}
$$ where $A_m$ is the sequence of actions
generated by line \ref{alg:normcs:sample} of \cref{alg:normcs};
and $\nu_n \in [0, (\tau - 1)^{-1})$ is a predictable
betting sequence (to be specified below).
This wealth process
is a non-negative martingale with
initial value of 1 when evaluated at
$\beta = \beta_1$ and therefore due to
Ville's inequality $$
\mathbb{P}\left(\forall n \in \mathbb{N}:
W^{(-)}_n(\nu; \beta) \leq \nicefrac{2}{\delta}\right) \geq 1 - \nicefrac{\delta}{2}.
$$
Because $g(\cdot, \beta)$ is non-decreasing
in $\beta$,
it follows $L_n = \sup \left\{ \beta \left| W_n^{(-)}(\nu; \beta) \leq \nicefrac{2}{\delta} \right. \right\}$ is a lower confidence sequence for $\beta_1$.  It remains to
specify the betting process $\nu_n$: we
use online Newton step to choose bets
that maximize the (log) wealth,
using loss $\left(-\log W^{(-)}_m(\cdot, L_{m-1})\right)$, and constraining the bet sequence
$\nu \in [0, \nicefrac{1}{2 \tau}]$ to ensure
a bounded gradient.

The upper bound is similar, but using a martingale of the form
$
W^{(+)}_n(v; \beta) = \prod_{m=1}^n \left(1 + v_m \left(g(A_m; \beta) - \kappa_{\infty}^{-1} \right)\right),
$ and constraining the bet sequence $v \in [0, \nicefrac{\kappa_{\infty}}{2}]$.

\section{Proof of Theorem~\ref{thm:normcs}}
\label{app:proofthmnormcs}

\thmNormCS*

This uses the notation from
\cref{app:explainnormcs}.

From
\citet[Corollary 1]{waudby2020estimating},
$L_n$ and  $U_n$ satisfy \cref{eqn:uppercs}
and  \cref{eqn:lowercs} respectively.
Therefore, whenever $L_n \geq U_n$,
given the monotonicity of $g(\cdot; \beta)$
wrt $\beta$, \cref{eqn:findbeta} holds with probability at least $(1 - \delta)$.

\section{Online Regret Experiment: Additional Details}
\label{app:onlineregretexperiment}
We perform the online regret experiment using twenty regression datasets hosted on OpenML~\citep{vanschoren2014openml} and released under a CC-BY\footnote{\url{https://creativecommons.org/licenses/by/2.0/}} license. The exact data ids for these datasets are: 150, 422, 1187, 41540, 41540, 42225, 42225, 44025, 44031, 44056, 44059, 44069, 44140, 44142, 44146, 44148, 44963, 44964, 44973, and 44977.

For large datasets a random subset of $80,000$ examples is selected. Features in every data set are transformed so that the $i$-th feature in sample $x_t^i$ is shifted by $\min_t x_t^i$ and scaled by $\nicefrac{1}{\max_t x_t^i -\min_t x_t^i }$. This transformation is applied to the labels $y_t$ as well so that for every label $y_t \in [0,1]$.

During evaluation of \smthigw and \cappedIGW the contexts $x_t$ is revealed to the learners in batches of eight. The learners then pick an action to play for each context in the batch. After picking their actions learners then receive the loss $\ell_t(a_t) = |a_t-y_t|_1$ for each of the selected actions. This process continues until all examples in a dataset are exhausted.

Both \smthigw and \cappedIGW assume access to a $\wh{f}_t$ and \smthigw also assumes access to an $\arg \min$ orcale to compute $\wh a_t$. To satisfy these requirements we mirror the implementation pattern of \citet{zhu2022contextual} where $\theta$ are learned parameters, $\wh f_t(x,a;\theta) := g(\wh a(x;\theta)-a;\theta)$, and $g$ is defined so that its global minimizer is $0$. With $\theta := (u;w;q;z;\zeta)$ our experiment defines $\wh a(x; \theta) = \sigma\left(u+\langle x, w \rangle \right)$ where $\sigma$ is the sigmoid function and, given $z = \wh a(x;\theta) - a$,
\begin{equation}
    \wh g(x,a;\theta) =
    \begin{cases}
     q + \langle w, (z, z^{3/2}, z^2) \rangle & \text{if}\ z \geq 0\\
     q + \langle \zeta, (|z|, |z|^{3/2}, |z|^2) \rangle & \text{if}\ z < 0.\\
    \end{cases}.\label{eq:argmax_plus_dispersion}
\end{equation}
To optimize $\theta$ we use a mean squared error loss with Adam \citep{kingma2014adam} in PyTorch \citep{Paszke_PyTorch_An_Imperative_2019}.

We use the Corral meta-algorithm from \citet{zhu2022contextual} to select the smoothness parameter $\tau$ for both \smthigw and \cappedIGW. The hyperparameter settings for the meta-algorithm were optimized and fixed globally to give the best average performance across all experiment datasets. For \smthigw we use $\eta := 0.3$ and select $\tau$ from the set $\{2, 3.76, 7.05, 13.24, 24.87, 46.7, 87.7, 164.69, 309.27, 580.77, 1090.6, 2048\}$. For \cappedIGW we use $\eta := 0.3$ and select $\tau$ from $\{6, 9.57, 15.28, 24.37, 38.89, 62.05, 99.01, 157.98, 252.08, 402.21, 641.77, 1024\}$.

\section{Offline Utility Experiment: Additional Details}
\label{app:offlineregretexperiment}

For the offline experiment we use data exhaust from the online experiment which takes the form of {$(x_t, a_t, \P_t(a_t), \ell_t(a_t))$}. Because of this the datasets for the offline experiment are identical to those in the online experiment as are the dataset transformations, both of which are described in \cref{app:onlineregretexperiment}.

The offline learners use the same functional form as the online learners; that is $f(x,a;\theta) := g(a^\star(x;\theta)-a;\theta)$ with the definition of $g$ given in \cref{eq:argmax_plus_dispersion}. The offline experiment uses a more complex form for $a^\star(x;\theta)$ than the online learners. Rather than one linear layer with a sigmoid output the offline learners use a three layer feedforward neural network with width equal to the number of features in a dataset, ReLU activation functions, and a sigmoid output.

We optimize offline learner parameters $\theta$ using mean squared error loss with Adam \citep{kingma2014adam} in PyTorch \citep{Paszke_PyTorch_An_Imperative_2019}. When using clipped IPS~\cite{strehl2010learning} we multiply each mean squared error by its importance weight. It is known that this is not an optimal way to perform importance updates \cite{karampatziakis2011online}. Even so, the offline learners still benefit from the importance weighted updates when using \cappedIGW exhaust. During testing our offline learners follow the policy $\pi^\star(x) := a^\star(x;\theta)$.

\section{\cappedIGW Sensitivity to \texorpdfstring{$\kappa_\infty$}{κ-∞}: Additional Details}
\label{app:sensitivityexperiment}

To further understand how $\kappa_\infty$ influences experimental outcomes we look here at the probabilities logged during online analysis along with the importance weights used during offline analysis. We see in \cref{fig:prob} that for our 20 datasets as $\kappa_\infty$ became larger logged \cappedIGW probabilities tended toward 0. In turn, we see a greater spread in offline importance weights from this data \cref{fig:weight}.

\begin{figure}[t]
\begin{minipage}[t]{0.48\linewidth}
\vskip 0pt
{
    \includegraphics[width=1\textwidth]{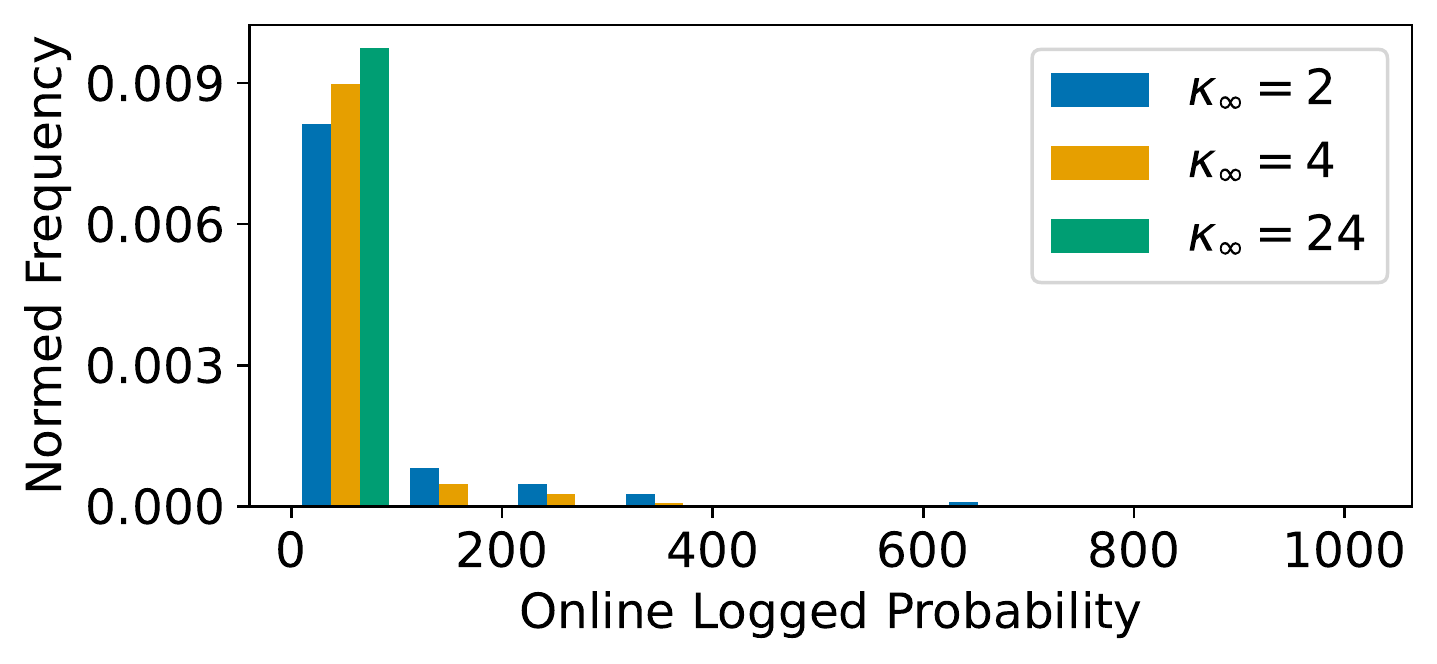}
    \vspace{-.5cm}
    \caption{As $\kappa_\infty$ grows online probability is increasingly near 0. \label{fig:prob}}
}
\vspace{-0.5cm}
\end{minipage}
\hfill
\begin{minipage}[t]{0.48\linewidth}
\vskip 0pt
{

    \includegraphics[width=1\textwidth]{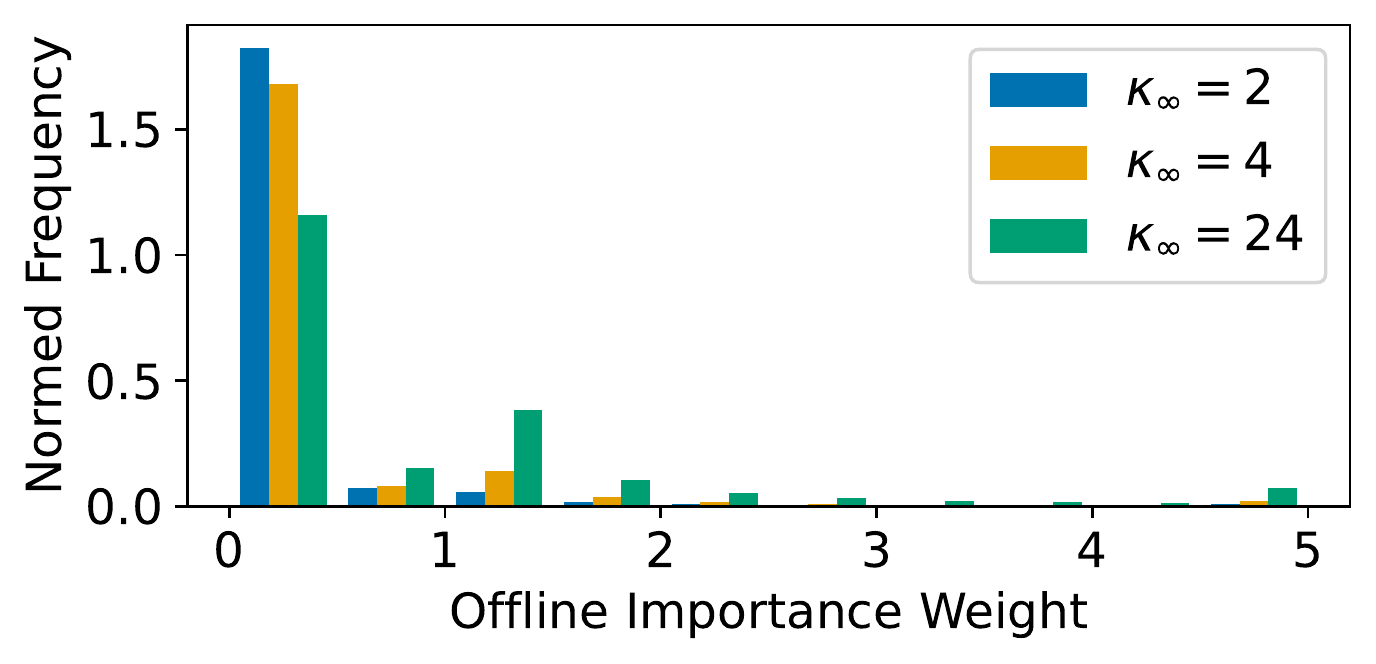}
    \vspace{-.67cm}
    \caption{As $\kappa_\infty$ grows offline weight has a greater spread. \label{fig:weight}}
}
\end{minipage}
\end{figure}

Another perspective can be found by looking at the distribution of actions played by our learners. For this we recorded the distance a learner's played action was from what the learner believed the greedy action was. This perspective is only useful in these experiments due to the unimodal implementation of $\hat f_t$ (see \cref{app:onlineregretexperiment}). We see that as $\kappa_\infty$ increases so to does the spread of actions played around the believed argmax \cref{fig:k_spread}.

\begin{figure}[t]
\includegraphics[width=.55\textwidth]{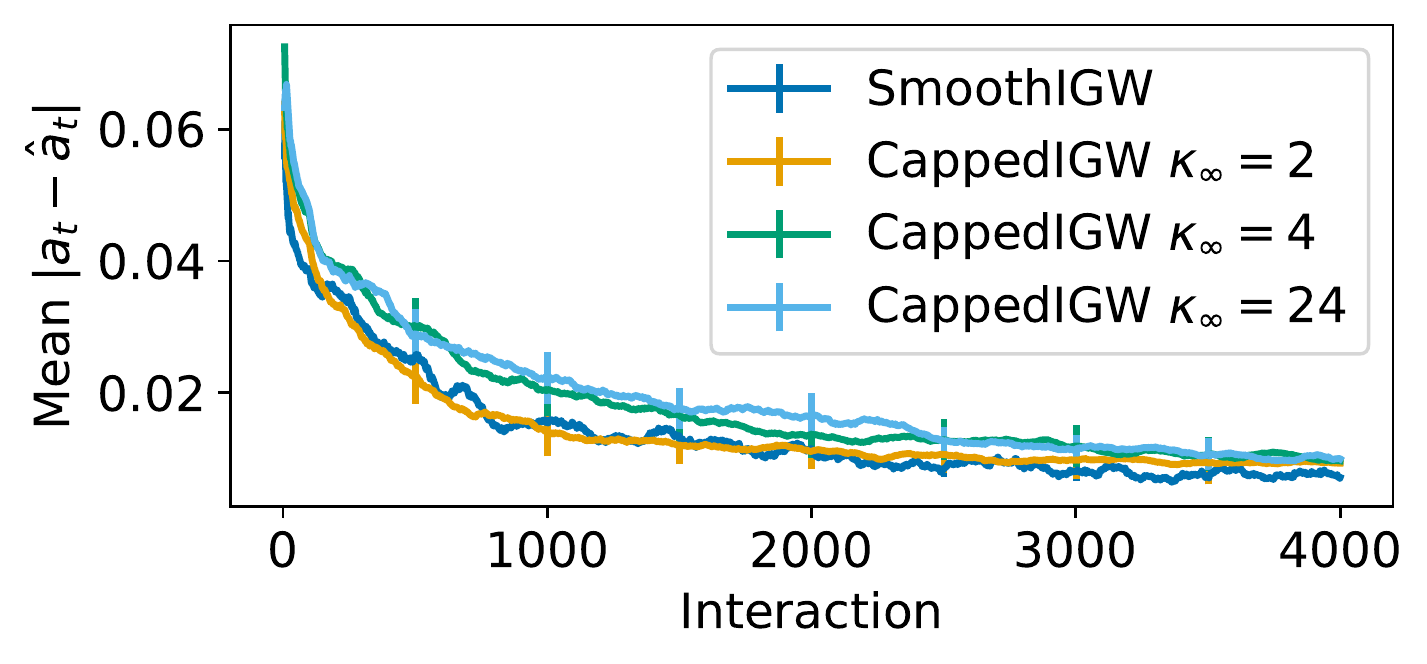}
\centering
\caption{As $\kappa_\infty$ grows actions have a greater spread around what the learners believes the argmax $\hat a_t$ is.\label{fig:k_spread}}
\end{figure}

\end{document}